\documentclass[11pt]{article}
\usepackage[margin=1in]{geometry}
\usepackage[T1]{fontenc}

\usepackage[dvipsnames]{xcolor}
\usepackage{algorithm}
\usepackage[noend]{algpseudocode}

\algnewcommand{\LComment}[1]{\Statex \textcolor{gray}{\(\triangleright\) \textit{#1}}}

\usepackage[hyphens]{url}
\usepackage[sort]{natbib}
\usepackage{rcx}
\usepackage{amssymb}

\usepackage{enumitem}
\usepackage{footnote}
\usepackage[bottom]{footmisc}
\usepackage{tablefootnote}
\usepackage{booktabs}
\usepackage{makecell}
\usepackage{subcaption}
\newcolumntype{C}[1]{>{\centering\arraybackslash}p{#1}}

\makeatletter
\newcommand{\phantomfill}[2]{{\mathpalette\mask@{{#1}{#2}}}}
\newcommand{\mask@}[2]{\mask@@{#1}#2}
\newcommand{\mask@@}[3]{%
  \settowidth{\dimen@}{$\m@th#1#2$}%
  \makebox[\dimen@]{$\m@th#1#3$}%
}
\makeatother

\usepackage{siunitx}[=v2]
\hypersetup{
  colorlinks   = true,
  urlcolor     = blue,
  linkcolor    = blue,
  citecolor    = blue
}

\newcommand{\LP}{\hyperlink{eq:primal}{\textnormal{LP}}}
\newcommand{\LPp}{\hyperlink{eq:primal}{\textnormal{Primal LP}}}
\newcommand{\LPd}{\hyperlink{eq:dual}{\textnormal{Dual LP}}}

\newcommand\dVC{d_\mathrm{VC}}
\newcommand\dP{d_\mathrm{P}}

\usepackage{authblk}

\begin{document}

\title{A Unified Post-Processing Framework for\\Group Fairness in Classification}
\date{}

\ifpdf
\renewcommand*{\Authand}{\qquad}
\author{Ruicheng Xian}
\author{Han Zhao}
\affil{University of Illinois Urbana-Champaign\protect\\{\small\texttt{\{\href{mailto:rxian2@illinois.edu}{rxian2},\href{mailto:hanzhao@illinois.edu}{hanzhao}\}@illinois.edu}}}
\else
\author[1]{Ruicheng Xian}
\author[2]{Han Zhao}
\affil[1]{University of Illinois Urbana-Champaign\protect\\{\small\texttt{\href{mailto:rxian2@illinois.edu}{rxian2@illinois.edu}}}}
\affil[2]{University of Illinois Urbana-Champaign\protect\\{\small\texttt{\href{mailto:hanzhao@illinois.edu}{hanzhao@illinois.edu}}}}
\fi

\maketitle

\begin{abstract}%
We present a post-processing algorithm for fair classification that covers group fairness criteria including statistical parity, equal opportunity, and equalized odds under a single framework, and is applicable to multiclass problems in both attribute-aware and attribute-blind settings.  Our algorithm, called ``LinearPost'', achieves fairness post-hoc by linearly transforming the predictions of the (unfair) base predictor with a ``fairness risk'' according to a weighted combination of the (predicted) group memberships.  It yields the Bayes optimal fair classifier if the base predictors being post-processed are Bayes optimal, otherwise, the resulting classifier may not be optimal, but fairness is guaranteed as long as the group membership predictor is multicalibrated.  The parameters of the post-processing can be efficiently computed and estimated from solving an empirical linear program.  Empirical evaluations demonstrate the advantage of our algorithm in the high fairness regime compared to existing post-processing and in-processing fair classification algorithms.\footnote{\label{fn:code}Our code is available at \url{https://github.com/uiuctml/fair-classification}.}
\end{abstract}

\section{Introduction}\label{sec:intro}
Algorithmic fairness is an important consideration in machine learning due to concerns that models may propagate social biases present in the training data against historically disadvantaged demographics~\citep{barocas2023FairnessMachineLearning}, especially when involving sensitive domains such as criminal justice, healthcare, and finance~\citep{barocas2016BigDataDisparate,berk2021FairnessCriminalJustice}. Work on this topic includes evaluation of unfairness, design of fair learning algorithms, and analysis of the inherent limitations in achieving fairness~\citep{kleinberg2017InherentTradeOffsFair,zhao2022inherent}.  To evaluate model unfairness, \textit{group fairness} metrics group individuals by pre-specified attributes (called \textit{sensitive attributes}, such as gender, race, etc.\@) and compute the difference between the aggregate statistics of the model output conditioned on the groups. Popular group fairness metrics include \textit{statistical parity}~(SP; \citealp{calders2009BuildingClassifiersIndependency}), \textit{equal opportunity}~(EOpp), and \textit{equalized odds}~(EO; \citealp{hardt2016EqualityOpportunitySupervised}); they differ in their specification of groups.

Fair learning algorithms can be broadly categorized into pre-processing, in-processing, and post-processing approaches, depending on where bias mitigation takes place in the training pipeline.  Pre-processing is applied to the (apparent) source of unfairness---the training data---by removing biased associations via data cleaning or reweighting techniques~\citep{kamiran2012DataPreprocessingTechniques,calmon2017OptimizedPreProcessingDiscrimination}; however, debiasing training data does not generally guarantee fairness for models trained on them.  In-processing modifies the model training objective to incorporate the fairness constraints, but necessitates specialized optimization procedures~\citep{zemel2013LearningFairRepresentations,zafar2017FairnessConstraintsMechanisms,agarwal2018ReductionsApproachFair}.  Post-processing~\citep{xian2023FairOptimalClassification,chzhen2020FairRegressionWasserstein,legouic2020ProjectionFairnessStatistical} is performed after the model is trained, by remapping model outputs to satisfy fairness (the remapping takes simpler form and is easier to train than the backbone model; in our case, it is a linear classifier on top of model outputs).  These approaches are not mutually exclusive and could be applied altogether, but options may be limited by the computational budget, or when parts of the training pipeline cannot be modified.

Among the three approaches above, post-processing is the most flexible and computationally lightweight as it is post-hoc, model-agnostic, and incurs minimal overhead. It is applicable to scenarios where the decision to perform bias mitigation is to be made after training and evaluating an initial iteration of the model, or if (re-)training subject to fairness constraints is expensive or not possible, such as when operating under the ``model as a service'' paradigm that utilizes large foundation models~\citep{sun2022BlackBoxTuningLanguageModelasaService,gan2023ModelasaServiceMaaSSurvey}.  Post-processing does not compromise on performance in the sense that it yields the Bayes optimal fair classifier when the predictors being post-processed are Bayes optimal score functions~\citep{zhao2022inherent,xian2023FairOptimalClassification} (which is not out of reach given the ability to train large models~\citep{ji2021EarlystoppedNeuralNetworks}).\footnote{We make the distinction between a score function and a classifier: the outputs of the former are probability estimates (scores, e.g., values in $[0,1]$ for binary classification), and the latter are class labels (e.g., $\{0,1\}$).}  Access to Bayes scores may appear to be a strong requirement for good post-processing performance. In particular, it is shown that if the hypothesis class $\calF$ is restricted and does not contain the Bayes scores, \citet{woodworth2017LearningNonDiscriminatoryPredictors} construct problem instances where post-processing underperforms in-processing.  But the empirical study of \citet{cruz2024UnprocessingSevenYears} on real-world data shows post-processing achieving better accuracy-fairness tradeoffs, suggesting that predictors learned in practice using off-the-shelf algorithms are adequate for post-processing.  Moreover, \citet{globus-harris2023MulticalibratedRegressionDownstream} show that post-processing performance matches in-processing as long as the base predictor satisfies a notion of \textit{multicalibration} with respect to $\calF$~\citep{hebert-johnson2018MulticalibrationCalibrationComputationallyIdentifiable,gopalan2022Omnipredictors}.

There is already a rich literature on post-processing algorithms for fair classification, but they are often stylized to specific fairness criteria and simplified problem settings, such as binary instead of multiclass, and \textit{attribute-aware} instead of \textit{attribute-blind} (i.e., whether the sensitive attribute is observed at test time).  We review some existing algorithms in \cref{sec:related}; notably among them, \citet{chen2024PosthocBiasScoring} recently propose a framework applicable to group fairness criteria covering SP, EOpp, and EO under both attribute-aware and blind settings, despite limited to binary classification.

Furthering this existing line of work, in this paper, we present a unified post-processing framework for fair classification that covers fairness criteria expressed as the difference of first-order group-conditional moments of the classifier output (\cref{def:parity}, which covers SP, EOpp, and EO), as well as all problem settings mentioned above (with attribute-blind and multiclass being the most general).  Our proposed algorithm maps the outputs of the base predictor (score functions) to class labels satisfying the required fairness criteria.

The paper is organized as follows. We provide problem setup and definitions, and summarize our main results in \cref{sec:summary}, followed by related work in \cref{sec:related}.  Then in \cref{sec:main}, we show that the optimal fair classifier can be expressed as a linear post-processing of Bayes optimal predictors (without fairness constraints).  Based on this, we present a fair post-processing algorithm in \cref{sec:standard}, and discuss best practices and how to instantiate it to SP, EOpp, and EO in attribute-aware and blind settings.  The sample complexity is analyzed, as well as the excess risk and unfairness of the classifier returned by our algorithm when the base predictors are not optimal.  We close by evaluating our algorithm on several datasets in \cref{sec:exp} and comparing to existing in-processing and post-processing algorithms.

\subsection{Preliminaries and Main Results}\label{sec:summary}

A classification problem is defined by a joint distribution over input features $X\in\calX$, labels $Y\in\calY$ (categorical), and group membership indicators $Z\in\{0,1\}^{K}$ for a set of $K$ groups, where $Z_k=1$ represents membership to group $k$ (the groups can be overlapping, that is, more than one group indicators can be active at a time).  Uppercase letters refer to random variables.

Let $r:\calX\times\calY\rightarrow[0,\infty)$ denote the non-negative \textit{pointwise risk}, where $r(x,y')$ is the risk of predicting class $y'$ on input $x$ (we may occasionally override notation and write $r(x)\equiv r(x,\cdot)\in[0,\infty)^\calY$). The overall risk of a classifier $h:\calX\rightarrow\calY$ is then $R(h)=\E[r(X, h(X))]$.  If the risk is computed with a loss function $\ell:\calY\times\calY\rightarrow [0,\infty)$, then the pointwise risk is $r(x,y')=\E[\ell(y',Y)\mid X=x]$.  For example, when $\ell$ is the 0-1 loss, then $r(x,y')$ is the (average) classification error on input $x$ from predicting $y'$:
\begin{equation}
  r(x,y')=\E[\1[Y\neq y']\mid X=x]=\Pr(Y\neq y'\mid X=x)=1-\Pr(Y=y'\mid X=x).\label{eq:0.1.loss}
\end{equation}
Evaluating the pointwise risk requires the conditional distribution of $\Pr(Y\mid X)$---called the \textit{Bayes optimal score function}---which is usually unknown and needs to be learned from the data, such as by fitting a predictor on labeled examples (without fairness constraints).

The Bayes optimal classifier, which achieves the minimum risk, is given by
\begin{equation}
  x\mapsto \argmin_{y\in\calY} r(x,y). \label{eq:optimal.unconst}
\end{equation}
However, this classifier may be unfair when biases are present in the underlying distribution (or the training data on which an $\hat r$ is learned), so instead we want to find the (potentially randomized) Bayes optimal fair classifier.
We consider fairness criteria that can be expressed as the difference of first-order group-conditional moments of the classifier output:

\begin{table}[t]
  \caption{Recovering commonly used group fairness criteria from \cref{def:parity} via the choice of the fairness constraints $\calC$ and group indicators $Z$. The sensitive attribute is denoted by $A\in\calA$.}
  \label{tab:equivalence}
  \centering
  \scalebox{0.85}{\begin{tabular}{lll}
      \toprule
      Criterion & Definition                                                           & \makecell[l]{Equivalent Choice of $\calC$ with\\$\phantomfill{\;\;Z_{a,y}}{\hfill Z_a}=\1[A=a]$\\$\;\;Z_{a,y}=\1[A=a,Y=y]$}                                                                    \\
      \midrule
      \makecell[l]{Statistical Parity
      }                  & $\begin{aligned}\displaystyle\begin{multlined}[t][3.65in]\max_{a,a'\in\calA} \big|\!\Pr(h(X)=y \mid A=a) \\[-1.7em] - \Pr(h(X)=y \mid A=a')\big|\leq \alpha, \; \phantomfill{\forall y,y'\in\calY}{\hfill\forall y\in\calY}\end{multlined}\end{aligned}$                                                     & $\{(y,\calA):y\in\calY\}$                    \\[1.7em]
      \makecell[l]{Equal Opportunity                                                                                                                                                                                                               \\(TPR Parity)
      }                  & $\begin{aligned}\displaystyle\begin{multlined}[t][3.65in]\max_{a,a'\in\calA} \big|\!\Pr(h(X)=1 \mid A=a,Y=1) \\[-1.7em] - \Pr(h(X)=1 \mid A=a',Y=1)\big|\leq \alpha\phantom{, \; \forall y,y'\in\calY}\end{multlined}\end{aligned}$ & $\{(1,\calA\times\{1\})\}$                                                                                           \\[1.7em]
      \makecell[l]{Multiclass\\Equal Opportunity
      }                  & $\begin{aligned}\displaystyle\begin{multlined}[t][3.65in]\max_{a,a'\in\calA} \big|\!\Pr(h(X)=y \mid A=a,Y=y) \\[-1.7em] - \Pr(h(X)=y \mid A=a',Y=y)\big|\leq \alpha, \; \phantomfill{\forall y,y'\in\calY}{\hfill\forall y\in\calY}\end{multlined}\end{aligned}$ & $\{(y,\calA\times\{y\}):y\in\calY\}$                      \\[1.7em]
      \makecell[l]{Equalized Odds
      }                  & $\begin{aligned}\displaystyle\begin{multlined}[t][3.65in]\max_{a,a'\in\calA} \big|\!\Pr(h(X)=y \mid A=a,Y=y') \\[-1.7em] - \Pr(h(X)=y \mid A=a',Y=y')\big|\leq \alpha, \; \forall y,y'\in\calY\end{multlined}\end{aligned}$                                                 & $\{(y,\calA\times\{y'\}):y,y'\in\calY\}$  \\
      \bottomrule
    \end{tabular}}
\end{table}

\begin{definition}[Group Fairness]\label{def:parity}
  Let the collection of fairness constraints be given in the form of $\calC=\{(y_1,\calI_1),(y_2,\calI_2),\dots,(y_C,\calI_C)\}$, where each set of constraints $(y_c,\calI_c)$ specifies a subset of (indices of) the group indicators $\calI_c\subseteq [K]\coloneqq\{1,\dots,K\}$ over which parity in output class $y_c\in\calY$ is evaluated/to be required.
  
  For tolerance $\alpha\in[0,1]$, a (randomized) classifier $h:\calX\rightarrow\calY$ satisfies $\alpha$-group fairness with respect to $\calC$ if
  \begin{align}
    \max_{k,k'\in\calI_c}\envert*{ \Pr(h(X)=y_c  \mid Z_k=1) - \Pr(h(X)=y_c \mid Z_{k'}=1) } \leq \alpha,\quad \forall c \in[C],
  \end{align}
  that is, the fraction of individuals belonging to group $k$ on which $h$ predicts $y_c$ differs from that of group $k'$ by no more than $\alpha$, for all $k,k'\in\calI_c$.
\end{definition}

\begin{remark}\label{rem:parity}
  \Cref{def:parity} covers the majority of commonly used group fairness criteria.  Let $A\in\calA$ denote the (categorical) sensitive attribute, which may represent gender, or race, etc.\ (to be distinguished from the generalized notion of group $Z$ in our definition).
  For instance, statistical parity, 
  \begin{align}
    \max_{a,a'\in\calA}\envert*{ \Pr(h(X)=y \mid A=a) - \Pr(h(X)=y \mid A=a') } \leq \alpha,\quad \forall y\in\calY,
  \end{align}
  is recovered by choosing $\calC=\{(y,\calA):y\in\calY\}$ with $Z_a=\1[A=a]$---indicators for $A$.  Equal opportunity and equalized odds, which are more complex, are recovered with group indicators of the form $Z_{a,y}=\1[A=a,Y=y]$: see \cref{tab:equivalence}.  However, \cref{def:parity} cannot recover the criterion of predictive parity (\citealp{chouldechova2017FairPredictionDisparate}; also considered as group-wise calibration), $Y\perp A\mid h(X)$, because here the classifier output $h(X)$ is in the conditioning part.
\end{remark}

\cref{def:parity} naturally handles overlapping groups and allows combining different fairness criteria.\footnote{Although, some criteria may not be compatible with each other.  For example, generally, exact SP and EO is only simultaneously achieved by constant classifiers~\citep{kleinberg2017InherentTradeOffsFair}.}

\paragraph{Main Results.}  Define the Bayes optimal group membership predictor $g:\calX\times[K]\rightarrow[0,1]$ by
\begin{equation}
  g(x,k)=\Pr(Z_k=1\mid X=x);\label{eq:constraint.fn}
\end{equation}
we may occasionally override notation and write $g(x)\equiv g(x,\cdot)\in[0,1]^K$ (or the $(K-1)$-dimensional simplex if each input belongs to exactly one group, which is the case for the fairness criteria in \cref{tab:equivalence}).

We show in \cref{thm:opt.fair} that, under a uniqueness assumption that can be satisfied via random perturbation, the Bayes optimal fair classifier---which achieves the minimal risk while satisfying the fairness constraints---takes the form of
\begin{equation}
    x\mapsto \argmin_{y\in\calY}\!\Bigg( r(x,y) + \underbrace{\sum_{k\in[K]} g(x,k)w(y,k) }_{\text{fairness risk}} \Bigg) \label{eq:0}
\end{equation}
for some weights $w$ depending on $\calC$ and $\alpha$.  This shows that the Bayes optimal fair classifier can be obtained from post-processing the Bayes optimal score: by offsetting the pointwise risk $r$ (which is a function of the score) with an additive ``fairness risk'' (similar in concept to the \textit{bias score} of~\citep{chen2024PosthocBiasScoring}).  
As a sanity check, when $w=0$, corresponding to not performing fair post-processing, it recovers the Bayes optimal classifier in \cref{eq:optimal.unconst}.  With post-processing, the weights $w(y,k)$ have values such that they discount the risk for individuals belonging to group $k$'s that would otherwise have a low rate of being labeled class $y$ to increase the rate, and vice versa. 

\Cref{eq:0} reveals two ingredients needed to perform fair post-processing: the group membership predictor $g$, and the appropriate weights $w$.  For SP in the attribute-aware setting, $g$ is given as the one-hot indicator for the explicitly provided sensitive attribute $A$, but in the attribute-blind setting, and for EOpp and EO where the group definition involves the label $Y$ which is not observed during test time, the group membership predictor needs to be learned.  As for the weights, we will show that they are given (and can be estimated from finite samples) by the optimal dual values of a linear program (\LP) representing the (empirical) fair post-processing problem.

Using these results,  in \cref{sec:standard}, we propose a fair classification algorithm based on post-processing.  The algorithm first estimates the pointwise risk $\hat r$ and group predictor $\hat g$ without fairness constraints using off-the-shelf learning algorithms---as plugins of the optimal $r,g$---then the weights $\hat w$ of the post-processing, and returns a classifier $\hat h$ of the same form of \cref{eq:0} but in terms of $(\hat r,\hat g,\hat w)$.  The excess risk and unfairness of $\hat h$ due to approximation and estimation error are analyzed in \cref{sec:sensitivity,sec:finite.sample}: notably, its unfairness is upper bounded by the multicalibration error of the predictor, and its excess risk by the $L^1$ distance of $\hat r$ from the optimal $r$. The excess risk is also more sensitive to the distance of $\hat r$ from $r$ when the fairness tolerance $\alpha$ is set to small values, providing another explanation for empirical observations that the post-processed classifier's accuracy tends to drop more rapidly in the high fairness regime, in addition to the inherent tradeoffs.

\subsection{Related Work}\label{sec:related}

This work considers achieving group fairness on classification problems with respect to fairness criteria that can be expressed as the difference of first-order group-conditional moments of the classifier output (\cref{def:parity}), including SP, EOpp, and EO\@.
There are other group-level criteria such as predictive parity~\citep{chouldechova2017FairPredictionDisparate} and ones defined using higher-order moments~\citep{woodworth2017LearningNonDiscriminatoryPredictors}, as well as individual-level metrics that look at how much model output varies on ``similar'' individuals~\citep{dwork2012FairnessAwareness}.  Further, \citet{hebert-johnson2018MulticalibrationCalibrationComputationallyIdentifiable} and \citet{kearns2018PreventingFairnessGerrymandering} consider the scenario where the groups are not pre-specified but only known to be generated from a class.\footnote{\label{fn:identify}That is, the indicators (or, in the attribute-blind setting, the Bayes optimal predictors) for the (unknown) groups of interest are assumed to be contained in a known (structured/finite) hypothesis class.}

\paragraph{In-Processing.}
Among in-processing algorithms for fair classification, two commonly used are fair representation learning~\citep{zemel2013LearningFairRepresentations,zhao2020ConditionalLearningFair} and the reductions approach~\citep{agarwal2018ReductionsApproachFair,yang2020FairnessOverlappingGroups}.  By training multilayer neural networks, fair representation learning algorithms learn feature representations at an intermediate layer that are distributionally invariant when conditioned on the groups (e.g., via adversarial training; \citealp{zhang2018MitigatingUnwantedBiases}), which is sufficient for satisfying the fairness constraint at the output layer.  The reduction approach analyzes the Lagrangian of the fair classification problem and turns it into a two-player game where the goal is to find the equilibrium; the algorithm returns a randomized ensemble of classifiers.

\paragraph{Post-Processing.}
Most existing post-processing algorithms for fair classification are stylized to specific fairness criteria and problem settings. The algorithms of \citet{jiang2020WassersteinFairClassification}, \citet{gaucher2023FairLearningWasserstein}, \citet{denis2023FairnessGuaranteeMulticlass}, and \citet{xian2023FairOptimalClassification} only support the SP criterion, whereas those of \citet{hardt2016EqualityOpportunitySupervised} and \citet{chzhen2019LeveragingLabeledUnlabeled} only support EOpp and EO in the attribute-aware setting.  Attribute-awareness is also required by the algorithms of \citet{menon2018CostFairnessBinary}, \citet{li2023FaiREEFairClassification}, and \citet{globus-harris2023MulticalibratedRegressionDownstream}. \citet{chen2024PosthocBiasScoring} and \citet{zeng2024BayesOptimalFairClassificationa} only consider binary classification.

Fair post-processing algorithms typically map the outputs of the base predictor directly to class labels. \citet{alghamdi2022AdultCOMPASFair} and \citet{tifrea2024FRAPPEGroupFairness} instead propose algorithms that output fair scores rather than class labels, via projecting the base predictor's output scores to satisfy fairness while minimizing the divergence between the projected and the original scores (e.g., KL divergence).  This objective can be shown to maximize the Brier score of the post-processed predictor~\citep{wei2020OptimizedScoreTransformation}, but it may not align with performance metrics such as accuracy.

\section{Optimal Fair Classifier as a Post-Processing}\label{sec:main}

Given pointwise risk $r:\calX\times\calY\rightarrow[0,\infty)$ (e.g., \cref{eq:0.1.loss}), Bayes optimal group predictor $g:\calX\times[K]\rightarrow[0,1]$ (see \cref{eq:constraint.fn}), fairness constraints  $\calC$ and tolerance $\alpha$ (\cref{def:parity}), we want to compute a (randomized) classifier $h:\calX\rightarrow\calY$ that is optimal for the problem of
\begin{equation}
  \min_{h:\calX\rightarrow\calY} \E[r(X,h(X))]\quad\textrm{subject to}\quad \textrm{$\alpha$-group fairness w.r.t.\ $\calC$}.  \label{eq:prob}
\end{equation}
This problem always has a (trivial) solution: any constant classifier (\cref{prop:constant}).

We show that, under a problem-dependent condition that can always be satisfied by randomly perturbing the pointwise risk (\cref{rem:continuity}; to be discussed in \cref{sec:continuity}), the Bayes optimal fair classifier can be represented as a post-processing of $r$ (\cref{ass:uniqueness} and \LPd\ will be presented in the next section):

\begin{theorem}[Representation]\label{thm:opt.fair}
  Under \cref{ass:uniqueness}, a minimizer of \cref{eq:prob} is given by
  \begin{equation}
    h(x) = \argmin_{y\in\calY}\!\Bigg(r(x,y) + \sum_{k\in[K]}  g(x,k)w(y,k)\Bigg),\quad w(y,k)= -\sum_{c\in[C]} \1[y_c=y, k\in\calI_c] \frac{\psi_{c,k}}{\Pr(Z_k=1)}
  \end{equation}
  (break ties to the minimizing class with the smallest index $y$), where $\{\psi_{c,k}:c\in[C], k\in\calI_c\}$ is a maximizer of $\LPd(r,g,\PP_X,\calC,\alpha)$.
\end{theorem}

As discussed earlier on \cref{eq:0}, this classifier adjusts the scores computed by $r(x)$ with an offset that can be interpreted as a ``fairness risk'', then outputs the class with the minimum risk.  We can view the post-processing as a linear classifier on the outputs of $r$ and $g$ as features, $(r(x),g(x))\in \RR^{\calY\times K}$; this linear classifier has coefficients $(\be_{y'}, w(y',\cdot))$ for each output class $y'\in\calY$, where $\be_{y'}\in\RR^\calY$ is the one-hot vector for $y'$.

In subsequent \cref{sec:derivation,sec:continuity}, we derive the form of the optimal fair classifier in \cref{thm:opt.fair}, and discuss \cref{ass:uniqueness} and how to satisfy it via random perturbation.  Then in \cref{sec:standard}, using these results, we propose a fair classification algorithm based on post-processing, describe how to efficiently estimate the weights $\hat w$ from finite samples, and analyze the sample complexity as well as the error from using inaccurate $\hat r,\hat g$ in practice as plugins for $r,g$.

\subsection{Deriving the Optimal Fair Classifier}\label{sec:derivation}

The representation result for the optimal fair classifier in \cref{thm:opt.fair} follows from formulating and analyzing \cref{eq:prob} as a linear program in terms of the pointwise risk $r$ and the group predictor $g$.

First, we represent the randomized classifier $h:\calX\rightarrow\calY$ in tabular form by a lookup table (a row-stochastic matrix), $\pi\in[0,1]^{\calX\times\calY}$, where row $x$ is the output distribution of $h$ given $x$, namely,
\begin{equation}\label{eq:coupling}
  \pi(x,y)\equiv \Pr(h(X)=y\mid X=x).
\end{equation}
Then, we can express both the objective and constraints of \cref{eq:prob} as inner products between $\pi$ and $r,g$, respectively, leading to a linear program with $\envert{\calX}\envert{\calY}+C$ variables and $\envert{\calX}+\sum_c\envert{\calI_c}$ constraints:
\begin{flalign}
  \quad\hypertarget{eq:primal}{\textrm{Primal LP:}} &&&&  &  & \min_{\mathclap{\pi\geq 0,q}}        \;\;      & \E_{X}\!\Bigg[\sum_{y\in\calY}  r(X,y)\pi(X,y)\Bigg]        \\
                                             &&&&  &        & \textrm{subject to}           \;\;    & \sum_{y\in\calY}\pi(x,y)  = 1, &  & \forall x\in\calX,                                                                                                                     \\
                                             &&&&  &       &                    & \envert*{\E_X\sbr*{\frac{g(X,k)}{\Pr(Z_k=1)}\pi(X,y_c)} - q_c  } \leq \frac\alpha2, &  & \forall c\in[C],\, k\in\calI_c. &&&
\end{flalign}
The first set of constraints is for the row-stochasticity of $\pi$, and the second is the fairness constraint for $\alpha$-group fairness with respect to $\calC$.
The variable $q_c$'s represent the centroids of the group-conditional probabilities for predicting class $y_c$, and are introduced to replace the pairwise comparisons (polynomial in $\envert{\calI_c}$ in total) with a linear number of comparisons to the centroid, since $|u_i-u_j|\leq\alpha$,  $\forall i,j$, if and only if $\exists q$ s.t.\ $|u_i-q|\leq \alpha/2$, $\forall i$.
Then the equivalence between the second constraint to group fairness follows directly from Bayes' rule: by \cref{eq:constraint.fn,eq:coupling},
\begin{align}
  \E_X\sbr*{ \frac{g(X,k)}{\Pr(Z_k=1)}\pi(X,y_c) }
   & =\int_\calX \Pr(X=x\mid Z_k=1) \Pr(h(X)=y_c\mid X=x) \dif x      \\
   & =\int_\calX \Pr(X=x\mid Z_k=1) \Pr(h(X)=y_c\mid X=x, Z_k=1) \dif x \\
   & =\int_\calX \Pr(h(X)=y_c, X=x\mid Z_k=1)\dif x                   \\
   & =\Pr(h(X)=y_c \mid Z_k=1); \label{eq:constraint.eq}
\end{align}
the second equality is because the distribution of $h(X)$ is fully determined given $X$.  The equivalence for the objective follows similarly.

The proof of \cref{thm:opt.fair} requires analyzing the dual problem of \LPp\ (derivation deferred to \cref{sec:proof}, which follows similar steps in deriving the dual of the Kantorovich optimal transport problem):
\begin{flalign}
\quad\hypertarget{eq:dual}{\textrm{Dual LP:}} &&&& &  & \max_{\mathclap{\phi,\psi}}       \;\;        & \E_X[\phi(X)] - \frac{\alpha}{2} \sum_{c\in[C]}\sum_{k\in\calI_c} \envert{\psi_{c,k}}                 \\
   &&&&  &        & \textrm{subject to}        \;\;       & \phi(x) +  \sum_{c:y_c=y}\sum_{k\in\calI_c} g(x,k) \frac{\psi_{c,k}}{\Pr(Z_k=1)} \leq r(x,y), &  & \forall x\in\calX,\,y\in\calY,                                                                                                                      \\
  &&&&   &        &                                & \sum_{k\in\calI_c} \psi_{c,k}=0,                                                                   &  & \forall c\in[C]. &&&
\end{flalign}

\begin{proof}[Proof of \cref{thm:opt.fair}]
Let $(\pi,q)$ be a minimizer of $\LPp(r,g,\PP_X,\calC,\alpha)$ and $(\phi,\psi)$ its corresponding dual values (maximizer of \LPd).
  By definition of \LP\ and the problem \cref{eq:prob} it represents, the randomized classifier that outputs $y$ on each $x$ with probability $\propto\pi(x,y)$ is an optimal classifier satisfying $\alpha$-group fairness with respect to $\calC$, and its risk is the optimal value of \LP\@.

Define 
\begin{equation}
  r_\mathrm{fair}(x,y) = r(x,y) + \sum_{k\in[K]}  g(x,k)w(y,k),\quad w(y,k)= -\sum_{c\in[C]} \1[y_c=y, k\in\calI_c] \frac{\psi_{c,k}}{\Pr(Z_k=1)},
\end{equation}
then the first constraint of \LPd\ reads $\phi(x)- r_\mathrm{fair}(x,y)\leq 0$ for all $x,y$.  By complementary slackness, $\pi(x,y)>0\iff \phi(x) - r_\mathrm{fair}(x,y)= 0$~\citep{papadimitriou1998CombinatorialOptimizationAlgorithms}, and it follows that
\begin{equation}
  \pi(x,y)>0\implies y\in\argmin_{y'\in\calY}r_\mathrm{fair}(x,y').\label{eq:1}
\end{equation}
Suppose not, then $\exists y'\in\calY$, $y'\neq y$, such that $r_\mathrm{fair}(x,y') - r_\mathrm{fair}(x,y) < 0$.  By the first constraint in \LPd, we know that $\phi(x)-r_\mathrm{fair}(x,y')\leq 0$.  Adding the two inequalities together, we get $\phi(x)- r_\mathrm{fair}(x,y) < 0$, which contradicts complementary slackness.

The function on the right-hand side of \cref{eq:1}, which has a simple form as a linear post-processing of $r$ and $g$, is what we want to show to represent the optimal fair classifier.  
\Cref{eq:1} says that if $\pi$---the minimizer of \LPp\ that represents the (randomized) optimal fair classifier in tabular form---has a non-zero probability of predicting class $y$ on input $x$, then $y$ is among the ``best'' classes that minimize $r_\mathrm{fair}(x,\cdot)$.
When $\pi$ randomizes its prediction among multiple classes on $x$, then there will be multiple classes that simultaneously minimize $r_\mathrm{fair}(x,\cdot)$, and in this case the $\argmin$ should break ties randomly according to $\pi(x,\cdot)$.  But because the values of $\pi$ cannot be recovered from $r_\mathrm{fair}$ alone, we do not know the distribution for randomizing the classes in cases of ties, so the r.h.s.\ of \cref{eq:1} does not generally contain enough information to represent $\pi$.

Yet, the r.h.s.\ of \cref{eq:1} is able to represent the optimal fair classifier $\pi$ if the ``best'' class is always unique $\forall x$, since there will be no ambiguity in $\argmin$, and the relation ``$\Longleftrightarrow$'' will hold in \cref{eq:1}.
We state this as an assumption for now---whether it holds or not depends on the data distribution and the fairness constraints (\cref{rem:uniqueness.violation})---and in the next section we show that this assumption can always be satisfied by randomly perturbing the pointwise risk $r$ (with incurring a small risk). This can be thought of as a way to inject randomness into an otherwise deterministic classifier, and is a standard practice as it is well-known that there are problem instances where (exact) group fairness is non-trivially achievable only by randomized classifiers.

\begin{assumption}[Uniqueness of Best Class]\label{ass:uniqueness}
  Given $r$ and $g$, for all $w$, the set of minimizers $\envert{\argmin_{y} \rbr{r(X,y) - \sum_{k} g(X,k)w(y,k)}}=1$  almost surely with respect to $X\sim\Pr_X$ (i.e., the set has cardinality one, which means that the minimizer is almost always unique).
\end{assumption}

By \cref{ass:uniqueness}, we have $\envert{\argmin_{y}r_\mathrm{fair}(x,y)}=1$ (almost surely), which together with \cref{eq:1} imply that $x\mapsto\argmin_{y}r_\mathrm{fair}(x,y)$ (break any possible ties to the smallest index) agrees with $\pi$ almost everywhere, and is therefore an optimal fair classifier.
\end{proof}

\subsection{Satisfying Assumption~\ref{ass:uniqueness} via Random Perturbation}\label{sec:continuity}

\Cref{ass:uniqueness}, required by \cref{thm:opt.fair} to represent the optimal fair classifier as a post-processing, can be satisfied by randomly perturbing the pointwise risk $r$ (proof is deferred to \cref{sec:proof}):
\begin{proposition}\label{rem:continuity}
  Let $\xi\in\RR^{\calY}$ denote an independent random variable with a continuous distribution and independent coordinates, then \cref{ass:uniqueness} is satisfied for $(r+\xi,g)$ with respect to the joint distribution of $(X,\xi)$.
\end{proposition}

This random perturbation strategy doubles as a natural and straightforward way of injecting randomness into the classifier under our construction (and also appeared in \citep{xian2023FairOptimalClassification}).  Note that randomization is generally required to achieve (exact) group fairness with non-trivial classifier accuracy~\citep{hardt2016EqualityOpportunitySupervised,woodworth2017LearningNonDiscriminatoryPredictors}; for instance, the reductions in-processing approach returns a randomized ensemble of classifiers~\citep{agarwal2018ReductionsApproachFair}.

We use it in our algorithm in \cref{sec:standard,sec:exp} to inject randomness into the otherwise deterministic post-processed classifiers $h$ in the form of \cref{thm:opt.fair}, so that their outputs depend on the realization of the noise $\xi$ upon each call to $h$:
  \begin{equation}
    h(x) = \argmin_{y\in\calY}\!\Bigg(r(x,y) + \xi_y + \sum_{k\in[K]}  g(x,k)w(y,k)\Bigg).
  \end{equation}

Random perturbation does not affect fairness, but will increase the risk proportional to the magnitude of $\xi$ (so the noise should be set as small as possibly allowed by machine precision).  In our \cref{sec:exp} experiments, we use uniformly distributed noise, $\xi\sim\calU(-\sigma,\sigma)^\calY$.  This choice incurs a risk of at most $\envert\calY\sigma/2$, by combining the sensitivity analysis in \cref{thm:sensitivity} (to be discussed in the next section) and the following bound on the $L^1$ distance between $r$ and $r+\xi$ w.r.t.\ the joint probability of $(X,\xi)$:
\begin{equation}
  \E_X\E_\xi\!\Bigg[\sum_{y\in\calY}\envert*{r(X,y)-(r(X,y)+\xi_y)}\Bigg]=\envert\calY\E\sbr{\envert{\calU(-\sigma,\sigma)}}=\frac{\envert\calY\sigma}2.
\end{equation}

\begin{remark}[Scenarios Where \cref{ass:uniqueness} Fails]\label{rem:uniqueness.violation}
\Cref{ass:uniqueness} fails when the push-forward $r\sharp\PP_X$ contains atoms and the randomized optimal fair classifier splits their mass.  This has been discussed in prior work on fair post-processing~\citep{gaucher2023FairLearningWasserstein,denis2023FairnessGuaranteeMulticlass,xian2023FairOptimalClassification,chen2024PosthocBiasScoring}, where it is typically handled by making a continuity assumption (\cref{ass:continuity}) on $r\sharp\PP_X$, that can otherwise be satisfied by smoothing it with a continuous distribution---the same mechanism underlying the proof of \cref{rem:continuity}.

\Cref{ass:uniqueness} may also fail when the fairness criteria is EO or EOpp in the attribute-aware setting and the optimal fair classifier must be randomized~\citep{hardt2016EqualityOpportunitySupervised,xian2023efficient}.
\end{remark}

\begin{remark}[Impact of Randomization on Individual Fairness]\label{rem:randomization}
A potential issue of randomized classifiers in practice is that individuals with the same feature  $x\in\calX$ may receive different predictions $y\in\calY$, which could be perceived as unfair on the individual-level.  To avoid this and obtain a deterministic classifier, one may resort to approximate fairness (e.g., under our framework, this could be done by setting $\xi=0$ and tuning $\alpha>0$).
\end{remark}


\section{Post-Processing for Group Fairness}\label{sec:standard}

We have shown in \cref{sec:main} that for group fairness criteria in \cref{def:parity}, the optimal fair classifier can be expressed as a linear post-processing of the pointwise risk $r$ in terms of the group predictor $g$, where the parameters of the post-processing can be found by solving the \LP\@.  By putting together the results in the preceding section, we present a fair classification algorithm based on post-processing.

\begin{algorithm}[t]
  \caption{LinearPost: Fair Classification via Post-Processing}
  \label{alg:post.proc}
  \begin{algorithmic}[1]
    \Require{Labeled samples $S_L$ of $(X,Y,Z)$, unlabeled samples $S_U$ of $X$, fairness constraints $\calC$ and tolerance $\alpha$ (\cref{def:parity}), and random noise $\xi\in\RR^\calY$}
    \LComment{Pre-training}
    \State Estimate the pointwise risk $\hat r:\calX\rightarrow\RR^\calY$ on $S_L$ \label{ln:1}
    \State Estimate the group predictor $\hat g:\calX\rightarrow[0,1]^K$ on $S_L$ \label{ln:2}
    \State (Optional) Calibrate the group predictor $\hat g$
    \LComment{Post-processing}
    \State Denote the empirical distribution of $X$ on $S_U$ by $\widehat\PP_X$
    \State Solve for the optimal dual values $\hat \psi$ of $\LP(\hat r + \xi,\hat g,\widehat\PP_X,\calC,\alpha)$ \label{ln:alg.lp}
    \State Define $\hat w(y,k)= -\sum_{c:y_c=y,k\in\calI_c} {\hat\psi_{c,k}}/{\widehat\Pr(Z_k=1)}$ and $\widehat \PP(Z_k=1)= \frac1{|S_U|}\sum_{x\in S_U} \hat g(x,k)$
    \State \Return $x \mapsto \argmin_{y}\rbr{\hat r(x,y) +\xi_y + \sum_{k}  \hat g(x,k)\hat w(y,k)}$ \label{ln:ret}
  \end{algorithmic}
\end{algorithm}

\paragraph{Algorithm.}
Given labeled samples $S_L=\{(x_i,y_i,z_i)\}_i$, unlabeled samples $S_U=\{x_i\}_{i=1}^N$, fairness constraints $\calC$ and tolerance $\alpha$, our algorithm for learning fair classifiers---outlined in \cref{alg:post.proc} and explained in detail below---consists of a pre-training stage followed by post-processing:
\begin{enumerate}
  \item\label{alg:step.1} \textit{Pre-Training.}~~If the pointwise risk $r$ or the group predictor $g$ is not provided, train these predictors on $S_L$ using any preferred off-the-shelf learning algorithm without imposing the fairness constraints. Denote the learned predictors by $\hat r$ and $\hat g$.
  
  We will say more in \cref{sec:ins} on what $r$ and $g$ would be when instantiated for the commonly used fairness criteria (\cref{tab:equivalence}) with the objective of minimizing the classification error.

  \item \textit{Calibration \textnormal{(Optional)}.}~~If high level of fairness is required, then calibrate the group predictor $\hat g$ (in the sense of \cref{eq:cal.0}, to be discussed later) using additional labeled data.  
The reason for this is because fairness generally cannot be achieved exactly with an inaccurate proxy for group membership, $\hat g\neq g$, but it is sufficient if the predictions made by $\hat g$ reflects the true underlying probabilities on average---that is, being calibrated.  We will demonstrate it in our sensitivity analysis in \cref{thm:calibration} and specify the calibration condition in our case.

  \item \textit{Post-Processing.}~~Let $\widehat \PP_X$ denote the empirical distribution of $X$ formed using the samples $S_U$, compute the optimal dual values $\hat \psi$ of the empirical $\LP(\hat r + \xi,\hat g,\widehat\PP_X,\calC,\alpha)$.

   Here, $\xi\in\RR^\calY$ is a continuous random noise independently sampled each time $\hat r$ is called, supplied by the user to implement the random perturbation strategy described in \cref{rem:continuity}---for satisfying the assumption in our representation result \cref{thm:opt.fair}.  
  In our experiments, we use uniform noise $\xi\sim\calU(-\sigma,\sigma)^\calY$ with $\sigma=0.0001 \cdot \widehat \E_X [\|\hat r(X)\|_{\infty}]$.

\item Return the randomized classifier (the randomness comes from the draw of $\xi$ upon each call)
  \begin{equation}
    \hat h(x) = \argmin_{y\in\calY}\!\Bigg(\hat r(x,y) +\xi_y + \sum_{k\in[K]}  \hat g(x,k)\hat w(y,k)\Bigg),
  \end{equation}
  where
  \begin{equation}
    \hat w(y,k)= -\sum_{c\in[C]} \1[y_c=y, k\in\calI_c]  \frac{\hat\psi_{c,k}}{\widehat\Pr(Z_k=1)}\quad\text{and}\quad\widehat \PP(Z_k=1)= \frac{1}{N}\sum_{x\in S_U} \hat g(x,k).
  \end{equation}
  
  The classifier $\hat h$ takes the same form as the optimal fair classifier in \cref{thm:opt.fair}, except for using $\hat r,\hat g$ as plugins for the Bayes optimal $r,g$, and $\hat w$ estimated from finite samples.
\end{enumerate}

We call our algorithm ``LinearPost'', as it learns fair classifiers via linearly transforming the outputs of the base predictors.
Note that the post-processing stage only requires unlabeled data---instead of the labels $Y,Z$, the algorithm uses $\hat r,\hat g$ as proxies for them.

In the following section, we describe instantiations of our algorithm to settings including the fairness criteria of SP, EOpp, and EO, with or without attribute-awareness.  Then in \cref{sec:sensitivity,sec:finite.sample}, we analyze the excess risk of the returned classifier $\hat h$ relative to the Bayes optimal fair classifier as well as its unfairness due to the suboptimality of $\hat r,\hat g$, and the estimation error of $\hat w$.

\subsection{Instantiations of the Post-Processing Framework}\label{sec:ins}

We instantiate our framework and show how to learn fair classifiers using \cref{alg:post.proc} with the objective of minimizing classification error and satisfying SP, EOpp, and EO (\cref{tab:equivalence}) in both attribute-aware and blind settings.  We also briefly discuss the form of the post-processed classifier.

\paragraph{0-1 Loss.}  When the performance metric is classification error, i.e., using the 0-1 loss $\ell(y', y)=\1[y'\neq y]$, as discussed in \cref{eq:0.1.loss}, the Bayes optimal pointwise risk is
\begin{equation}
  r(x, y')=1-\Pr(Y=y'\mid X=x) = \Pr(Y\neq y'\mid X=x),
\end{equation}
so that $\E[r(x,h(x))]=\Pr(h(X)\neq Y)$.

This means that on \cref{ln:1} of \cref{alg:post.proc}, learning $\hat r$ amounts to training a predictor $\hat f_Y:\calX\rightarrow\Delta^\calY$ for the label $Y$ given $X$, $\hat f_Y(x)_{y'}=\widehat\Pr(Y=y'\mid X=x)$ (called a score function; $\Delta^\calY$ denotes the probability simplex over $\calY$), and set $\hat r(x,y')=1-\hat f_Y(x)_{y'}$.

\paragraph{Statistical Parity.}  To recover statistical parity from our group fairness definition in \cref{def:parity}, we set $Z\in\{0,1\}^\calA$ to be the one-hot indicator for the sensitive attribute $A\in\calA$, and the constraints $\calC$ according to \cref{tab:equivalence}.

In attribute-aware setting, there is no need to learn the group predictor on \cref{ln:2} of \cref{alg:post.proc} because $A$ is always explicitly given: $g(x,a)=\1[A=a\mid X=x]$.   In the attribute-blind setting, the Bayes optimal group predictor is $g(x,a) = \Pr(A=a\mid X=x)$; this entails learning a predictor for $A$ given $X$, $\hat f_A:\calX\rightarrow\Delta^\calA$, and set $\hat g(x,a)=\hat f_A(x)_a$.

With $\hat r$ and $\hat g$ learned, the algorithm  proceeds to estimating the weights of the post-processing via solving for the optimal dual values of the empirical $\LP(\hat r + \xi,\hat g,\widehat\PP_X,\calC,\alpha)$, returning a randomized classifier on \cref{ln:ret}.  In the binary class attribute-aware setting, the returned classifier is a (noisy) group-wise thresholding of the score function, which is known to represent the Bayes optimal fair classifier~\citep{menon2018CostFairnessBinary}:
\begin{align}
  \hat h(x,a) 
  &= \1\sbr*{\widehat\Pr(Y=1\mid X=x)  + \frac12\rbr*{\xi_0-\xi_1} \geq \hat t_a}\quad\text{where}\quad \hat t_a=\frac12\rbr*{1 + \hat w(0,a) - \hat w(1,a)}
\end{align}
(note that here, attribute-awareness is reflected by $\hat h$ taking the sensitive attribute $a$ as an explicit input).
In the multiclass setting, the notion of thresholding is generalized to class offsets~\citep{xian2023FairOptimalClassification}: the output is the class with minimum pointwise risk after being offset with the fairness risk as discussed in \cref{eq:0}.  In fact, for attribute-aware SP, by grouping the inputs $X$ according to the sensitive attribute $A$, \LPp\ can be rewritten as a Wasserstein barycenter problem over the group-conditional distributions of the pointwise risk: see LP of \citealp{xian2023FairOptimalClassification}.

In the attribute-blind setting, since $A$ is unknown, the threshold is dynamic according to the demographic makeup at each $x$ according to $\hat g$:
\begin{align}
  \hat h(x) 
  &= \1\sbr*{\widehat\Pr(Y=1\mid X=x)  + \frac12\rbr*{\xi_0-\xi_1} \geq \sum_{a\in\calA} \widehat\Pr(A=a\mid X=x) \hat t_a}.
\end{align}

\paragraph{Equalized Odds.} 
 To recover equalized odds from \cref{def:parity}, we set $Z\in\{0,1\}^{\calA\times\calY}$ to be the indicator for $(A,Y)$ jointly, that is, $Z_{a,y}=\1[A=a,Y=y]$.  Equal opportunity is not separately discussed because it can be seen as a special case of EO\@.
 
The Bayes optimal group predictor in this case is
\begin{equation}
  g(x,(a,y))=\Pr(A=a, Y=y\mid X=x) = \Pr(Y=y\mid A=a, X=x) \Pr(A=a\mid X=x).
\end{equation}
In the attribute-aware setting, because we already have a predictor $\hat f_Y$ for $Y$ from \cref{ln:1} (assuming that it makes use of $A$ for prediction) and $\Pr(A=a\mid X=x)$ is known thanks to attribute-awareness (as in attribute-aware SP), we do not need to learn any additional predictor for $\hat g$ according to the decomposition.  In the attribute-blind setting, we need to learn a predictor $\hat f_{A,Y}:\calX\rightarrow\Delta^{\calA\times\calY}$ for the joint $(A,Y)$.  This can be done directly, or, one could use the above decomposition to break the problem down into learning two (simpler) predictors, $\hat f_{Y}:\calA\times\calX\rightarrow\calY$ and $\hat f_{A}:\calX\rightarrow\calA$, then combining them into $\hat f_{A,Y}(a,y)=\hat f_{Y}(a,x)_y f_{A}(x)_a$.  An alternative decomposition, used in~\citep{alghamdi2022AdultCOMPASFair}, is  $\Pr(A,Y\mid X)=\Pr(A\mid Y, X)\Pr(Y\mid X)$; the choice of one over the other (or neither) depends on whether $A$ or $Y$ can be naturally integrated with the input features $X$.

The classifier returned by \cref{alg:post.proc} for EO, in the binary class attribute-aware setting, is a group-wise thresholding of the score function with a scalable noise:
\begin{align}
  \hat h(x,a) 
  &= \1\sbr*{\widehat \Pr(Y=1\mid X=x)  + \hat s_a (\xi_0-\xi_1) \geq \hat t_a },
\end{align}
where
\begin{align}
 \hat t_a &= \hat s_a \rbr*{1+\hat w(0,(a,0))-\hat w(1,(a,0))},\\
 \hat s_a &= \rbr*{2+\hat w(0,(a,0))-\hat w(1,(a,0))+\hat w(1,(a,1))-\hat w(0,(a,1))}^{-1}.
\end{align}
For this setting, \citet{hardt2016EqualityOpportunitySupervised} have shown that the optimal fair classifier is a group-wise thresholding with thresholds chosen so that the group-conditional TPRs and FPRs coincide across groups; if the conditional ROCs do not intersect except at the end points, randomization can be applied on the dominating group to worsen its ROC to intersect that of the other group.  These properties are embodied in our returned classifier: if the ROCs do not originally intersect, $\hat s_a$ amplifies the magnitude of the noise added to the score function on the dominating group thereby worsening its predictive power and its ROC\@.

\paragraph{Overlapping Groups}

So far, we have only discussed the case where the groups are disjoint, but our framework naturally handles overlapping groups, that is, $Z_k$ can be active for more than one $k\in[K]$ at a time.

Going from disjoint to overlapping groups only changes how the group predictor $\hat g$ should be learned: in the disjoint case, $g:\calX\rightarrow\Delta^{K-1}$ is for predicting a single group label at a time, whereas in the overlapping case, $g:\calX\rightarrow[0,1]^K$ becomes a multilabel predictor because $x$ could belong to multiple groups.  One may directly use the multilabel formulation, or alternatively, learn binary predictors $\hat f_{Z_k}:\calX\rightarrow[0,1]$ for each group $k$ separately.

\paragraph{Intersecting Groups}  There are also cases where fairness is required across subgroups formed by taking the intersection of groups, such as the combinations of multiple sensitive attributes like gender and race~\citep{buolamwini2018GenderShadesIntersectional}, $(A_1,A_2,\dots,A_M)$.  Our framework can handle intersecting groups---simply by representing each intersectional group by an indicator $Z_k=\1[A_1=a_1,\dots,A_M=a_M]$, but the number of indicators and fairness constraints grow exponentially in $M$, as well as the parameters of the post-processing.

It is possible to get around this exponential complexity by using alternative algorithms to optimize the post-processing (rather than solve the \LP\ directly). For example, the in-processing algorithm by \citet{kearns2018PreventingFairnessGerrymandering} views (the Lagrangian of) the fair classification problem of \cref{eq:prob} as a two-player zero-sum game between the learner $h$ and the auditor (dual variables), and uses an iterative algorithm to find an approximate equilibrium of the game.  It returns a randomized ensemble of classifiers, where the number of classifiers is equal to the iterations---polynomial in the optimality tolerance.  Such algorithms could be adapted to optimize our post-processing problem (\LP), which would build the post-processing up in an iterative manner (see, e.g., \citealp{globus-harris2023MulticalibratedRegressionDownstream}); we leave it as future work.

\subsection{Sensitivity Analysis}\label{sec:sensitivity}

With pointwise risk $\hat r$ or group predictor $\hat g$ that are not Bayes optimal, \cref{alg:post.proc} uses them as plugins for the optimal $r,g$ and applies \cref{thm:opt.fair} to obtain a post-processed classifier of the same form as the optimal fair classifier, using parameters $\hat w$ given by the optimal dual values of $\LP(\hat r,\hat g,\PP_X,\calC,\alpha)$.  The resulting classifier may not be optimal nor satisfy fairness exactly, but we can bound them in terms of the suboptimality of  $\hat r$ and $\hat g$ as follows (proof deferred to \cref{sec:proof.sensitivity}):

\begin{theorem}[Sensitivity]\label{thm:sensitivity}
  Given $r$ and $g$, denote the optimal value of $\LP(r,g,\PP_X,\calC,\alpha)$ by $\mathrm{OPT}$.
  \begin{enumerate}
    \item \label{thm:sensitivity.1} For any $\hat r\neq r$, define
    \begin{equation}
      \varepsilon_{\hat r}=\E_X \!\Bigg[\sum_{y\in\calY} \envert*{\hat r(X,y)-r(X,y)}\Bigg].
    \end{equation}

    Let $\hat\pi$ be a minimizer of $\LPp(\hat r,g,\PP_X,\calC,\alpha)$, its excess risk is
    \begin{equation}
       \E_X\!\Bigg[ \sum_{y\in\calY} r(X,y)\hat\pi(X,y)\Bigg]\leq \mathrm{OPT} + \varepsilon_{\hat r}.
    \end{equation}

    \item \label{thm:sensitivity.2} For any $\hat g\neq g$, define
          \begin{equation}
            \varepsilon_{\hat g}=\max_{k\in [K]}\E_X\sbr*{\envert*{\frac{\hat g(X,k)}{\E_X[\hat g(X,k)]}-\frac{g(X,k)}{\E_X[g(X,k)]}}}\leq \max_{k\in [K]}\frac{2\E_X\sbr*{\envert*{\hat g(X,k)- g(X,k)}}}{\E_X[g(X,k)]}.\label{eq:err.g}
          \end{equation}

    Let $\hat\pi$ be a minimizer of $\LPp(r,\hat g,\PP_X,\calC,\alpha)$, its excess risk is
          \begin{equation}
             \E_X\!\Bigg[ \sum_{y\in\calY} r(X,y)\hat\pi(X,y)\Bigg] \leq \mathrm{OPT} + 2\enVert{r}_\infty\min\rbr*{1-\alpha,\frac{\varepsilon_{\hat g}}{\alpha+\varepsilon_{\hat g}}}
          \end{equation}
          with the convention $0/0\coloneqq0$, and its constraint violation is
          \begin{equation}
            \max_{k,k'\in\calI_c}\envert*{\E_X \sbr*{\rbr*{\frac{g(X,k)}{\E_X[g(X,k)]}-\frac{g(X,k')}{\E_X[g(X,k')]}} \hat\pi(X, y_c)  }} \leq \min(1,\alpha + \varepsilon_{\hat g}),\quad\forall c\in[C].\label{eq:33a}
          \end{equation}
  \end{enumerate}
\end{theorem}

For generality, the results here are stated for the randomized classifier in tabular form $\hat\pi$; it is equivalent to the classifier $\hat h$ returned from \cref{alg:post.proc} when $\xi=0$ and \cref{ass:uniqueness} is satisfied.
If $r,g$ in the statement of \cref{thm:sensitivity} are Bayes optimal, then $\mathrm{OPT}$ is the risk of the optimal fair classifier, and $\E_X[g(X,k)]=\Pr(Z_k=1)$ and \cref{eq:33a} is equivalent to the violation of the fairness constraint $c$ by $\hat\pi$ (recall the definition of \LPp).

The first part considers the case where the pointwise risk $\hat r$ deviates from the optimal $r$, and bounds the excess risk of $\hat\pi$ obtained according to $\hat r$ by their $L^1(\PP_X)$ distance.  Fairness is satisfied by $\hat\pi$ because the group predictor $g$ is the Bayes optimal one in this case.

The second part considers the deviation of the group predictor $\hat g$, and bounds the excess risk and unfairness of $\hat\pi$ obtained according to $\hat g$ by its $L^1(\PP_X)$ distance from the optimal $g$. Unlike the first part, however, the excess risk here depends on the specified fairness tolerance by a factor of $1/\alpha$, meaning that the excess risk could be large in the high fairness regime if $\hat g$ is not accurate.  
This means that the accuracy of the post-processed classifier is expected to drop more rapidly when the fairness tolerance $\alpha$ is set to small values.  We remark that this dependency on $1/\alpha$ is tight in the sense that problem instances can be constructed to match the bound up to a constant factor: \cref{ex:tightness} in \cref{sec:proof.sensitivity} provides one such example from post-processing for statistical parity in the attribute-blind setting.

The bound on unfairness in the second part can be refined to the multicalibration error of each component of the group predictor, $\hat g_k$, $k\in[K]$, with respect to the joint level sets of the pointwise risk $r$ and the remaining components $\hat g_{-k}$~\citep{hebert-johnson2018MulticalibrationCalibrationComputationallyIdentifiable}:

\begin{corollary}[Multicalibration]\label{thm:calibration}
  Continuing \cref{thm:sensitivity} Part~\ref{thm:sensitivity.2}, define
  \begin{align}
    \varepsilon^\mathrm{cal}_{\hat g} 
    &=\max_{k\in[K]}\sum_{S\in\calS(r,\hat g)} \envert*{ \E_X\sbr*{\rbr*{  \frac{\hat g(X,k)}{\E_X[\hat g(X,k)]}-\frac{g(X,k)}{\E_X[g(X,k)]}  } \1[X\in S] }}
    \\ 
    &\leq \max_{k\in[K]} \frac{2}{\E_X[g(X,k)]} \underbrace{\sum_{S\in\calS(r,\hat g)} \big|\!\E_X\sbr*{\rbr*{\hat g(X,k)-g(X,k) } \1[X\in S] } \big| }_{\text{multicalibration error of $\hat g_k$}}, \label{eq:cal.0}
  \end{align}
  where $\calS(r,\hat g)=\{ \{x\in\calX: r(x)=u,\ \hat g(x)=v \} : u\in\RR^\calY,\ v\in\RR^K \}$.

  Let $\hat \pi$ be a minimizer of $\LPp(r,\hat g,\PP_X,\calC,\alpha)$, and define
  \begin{equation}
    \tilde \pi(x,\cdot) = \E_X[\hat \pi (X,\cdot)\mid r(X)=u, \hat g(X)=v ],\quad\text{where } u=r(x),\, v = \hat g(x),  \label{eq:same}
  \end{equation}
  which is also a minimizer; it has the same output distribution on all $x$ that map to the same $(r(x),\hat g(x))$ features, and is equal to $\hat\pi$ almost everywhere under \cref{ass:uniqueness}. Its constraint violation is
    \begin{equation}
      \max_{k,k'\in\calI_c}\envert*{\E_X \sbr*{\rbr*{\frac{g(X,k)}{\E_X[g(X,k)]}-\frac{g(X,k')}{\E_X[g(X,k')]}} \tilde\pi(X, y_c)  }} \leq \min(1,\alpha + \varepsilon^\mathrm{cal}_{\hat g}),\quad\forall c\in[C].
    \end{equation}
\end{corollary}

So even if the group predictor is not Bayes optimal, $\hat g\neq g$, as long as it is multicalibrated (in the sense of \cref{eq:cal.0}), the classifier post-processed according to $\hat g$ can achieve fairness exactly. Our empirical results in \cref{sec:exp} show that higher levels of fairness can indeed be achieved using calibrated group predictors.  \citet{globus-harris2023MulticalibratedRegressionDownstream} show a similar multicalibration condition in the binary class setting, and a related result by \citet{diana2022MultiaccurateProxiesDownstream} says that multicalibrated group predictors are sufficient proxies of the true group membership for learning optimal fair classifiers via in-processing.

\subsection{Finite Sample Estimation}\label{sec:finite.sample}

Given pointwise risk $r$ and group predictor $g$, to obtain a fair classifier of the form in \cref{thm:opt.fair}, we need to find the weights $w$ for satisfying fairness constraints via solving $\LP(r,g,\PP_X,\calC,\alpha)$ defined over the population $\PP_X$.

If only (unlabeled) finite samples $x_1,\dots,x_N\sim \PP_X$ are available, we may use the empirical distribution $\widehat \PP_X=\frac1N\sum_{i=1}^N\delta_{x_i}$ supported on the samples in place of $\PP_X$, solve the empirical $\LP(r,g,\widehat\PP_X,\calC,\alpha)$ (essentially replacing $\EE_X$ in \LP\ by the average over the samples), and use its optimal dual values $\hat\psi$ to estimate the weights $\hat w$ for post-processing---this is how finite sample estimation is performed in the post-processing stage of \cref{alg:post.proc}.  The excess risk and unfairness of the resulting classifier due to this estimation is bounded as follows  (proof deferred to \cref{sec:proof.sample}):

\begin{theorem}[Sample Complexity]\label{thm:sample}
  Given $r,g$, denote the optimal value of $\LP(r,g,\PP_X,\calC,\alpha)$ by $\mathrm{OPT}$.  Let $\widehat\PP_X$ denote the empirical distribution over i.i.d.\ samples $x_1,\dots,x_N\sim\PP_X$ (that do not depend on $r,g$).  

    Let $\hat\psi$ be a maximizer of $\LPd(r,g,\widehat\PP_X,\calC,\alpha)$, define $\hat h(x) =\argmin_{y}(r(x,y) + \sum_{k}  g(x,k) \hat w(y,k))$ where $\hat w(y,k)= -\sum_{c:y_c=y,k\in\calI_c} {\hat\psi_{c,k}}/{\widehat\Pr(Z_k=1)}$ and $\widehat \PP(Z_k=1)= \frac1{N}\sum_{i} g(x_i,k)$.     
    Then, under \cref{ass:uniqueness}, for all $N\geq \max(\envert\calY, \Omega(\max_{k}  \ln(K/\delta)/\Pr(Z_k=1)^2))$, with probability at least $1-\delta$, its excess risk is
  \begin{align}
    \MoveEqLeft\E_X\sbr{r(X,\hat h(X))} \\
    &\leq \mathrm{OPT} + \enVert  r_\infty O\rbr*{ \max_{k\in[K]} \frac{1}{\alpha \Pr(Z_k=1)} \sqrt{\frac{\ln K/\delta}{N}} + \envert\calY   \sqrt{\frac{\envert\calY K \log \envert\calY + \ln K/\delta}{N}}  },
  \end{align}
  and its constraint violation is
  \begin{align}
    \MoveEqLeft\max_{k,k'\in\calI_c}\envert*{\Pr(\hat h(X)=y_c\mid Z_k=1) - \Pr(\hat h(X)=y_c\mid Z_{k'}=1)  } \\
    &\leq \alpha + O\rbr*{ \max_{k\in[K]}\frac{1}{\Pr(Z_k=1)} \sqrt{\frac{\envert\calY K \log \envert\calY + \ln K/\delta}{N}} },\quad \forall c\in[C].
  \end{align}
\end{theorem}

Again, when $r,g$ in the statement are the Bayes optimal pointwise risk and group predictor, $\mathrm{OPT}$ is the risk of the optimal fair classifier and the second result bounds the unfairness of $\hat h$.
Because the fairness criteria cannot be evaluated exactly in the empirical $\LP(r,g,\widehat\PP_X,\calC,\alpha)$ from finite samples, the excess risk also has a $1/\alpha$ dependency on the specified fairness tolerance, akin to the case of using an inaccurate $\hat g$ in \cref{thm:sensitivity} Part~\ref{thm:sensitivity.2}.

The proof involves a uniform convergence bound and analyzing the complexity of $\hat h$, which is a linear multiclass classifier on the $(\calY\times K)$-dimensional features computed by $(r,g)$ (we do not consider the complexities of $r,g$ and assume them to be given).
In particular, the $\envert\calY K\log{\envert\calY}$ term comes from the VC dimension of the linear multiclass classifier in one-versus-rest mode (\cref{prop:vc.ovr}; note that this is more complex than linear binary classifiers, which are half-spaces); the extra factor of $\envert\calY$ in the risk bound is because we compute the overall risk by summing over the risks of one-versus-rest classifiers over $y\in\calY$.

\section{Experiments}\label{sec:exp}

We evaluate our proposed post-processing fair classification algorithm (\cref{alg:post.proc}) and compare to existing post-processing and in-processing algorithms.  Further details on the setup and hyperparameters can be found in our code repository.\footref{fn:code}

\paragraph{Tasks.}  The performance metric is classification error (0-1 loss), and the fairness criteria are SP, EOpp, or EO  (\cref{tab:equivalence}).  We consider the more general attribute-blind setting.

\begin{itemize}
  \item \textit{Adult~\textnormal{\citep{kohavi1996ScalingAccuracyNaiveBayes}}.}~~The task is to predict whether the income is over \$50k, and the sensitive attribute is gender ($\envert\calY=\envert\calA=2$).  It contains 48,842 examples.

  \item \textit{COMPAS~\textnormal{\citep{angwin2016MachineBias}}.}~~The task is to predict recidivism,  the sensitive attribute is race (limited to African-American and Caucasian; $\envert\calY=\envert\calA=2$).  It contains 5,278 examples.

  \item \textit{ACSIncome~\textnormal{\citep{ding2021RetiringAdultNew}}.}~~This dataset extends the Adult dataset, containing much more examples (1,664,500 in total).  We consider a binary setting identical that of Adult (ACSIncome2; $\envert\calY=\envert\calA=2$), as well as a multi-group multiclass setting with race as the sensitive attribute and five income buckets (ACSIncome5; $\envert\calY=\envert\calA=5$).

  \item \textit{BiasBios~\textnormal{\citep{de-arteaga2019BiasBiosCase}}.}~~This is a text classification task of identifying the job title from biographies; the sensitive attribute is gender ($\envert\calY=28$, $\envert\calA=2$).  We use the version scrapped and hosted by~\citet{ravfogel2020NullItOut}, with 393,423 examples.
\end{itemize}
We split Adult and COMPAS datasets by 0.35/0.35/0.3 for pre-training, post-processing, and testing, respectively, ACSIncome by 0.63/0.07/0.3, and BiasBios by 0.56/0.14/0.3.  For in-processing algorithms, the pre-training and post-processing splits are combined to be used for training.

\paragraph{Algorithms.}  We evaluate post-processing and in-processing fair classification algorithms that are applicable in the attribute-blind setting.  These algorithms are described in \cref{sec:related}.
Post-processing algorithms include:
\begin{itemize}
\item \textit{LinearPost~\textnormal{(Ours)}.}~~Applied and instantiated exactly as described in \cref{sec:standard}.  The empirical linear program on \cref{ln:alg.lp} of \cref{alg:post.proc} is solved using the Gurobi optimizer.\footnote{\url{https://www.gurobi.com}.}

  \item \textit{MBS~\textnormal{\citep{chen2024PosthocBiasScoring}}.}\footnote{\url{https://github.com/chenw20/BiasScore}.}~~Limited to binary classification.  This algorithm returns classifiers of forms similar to ours in \cref{thm:opt.fair}, except that they optimize for the post-processing weights by (heuristic) grid search on labeled data, hence their running time is exponential in the number of constraints.  We limit the number of grids to $M=5000$.

  \item \textit{FairProjection~\textnormal{\citep{alghamdi2022AdultCOMPASFair}}.}\footnote{\url{https://github.com/HsiangHsu/Fair-Projection}.}~~This algorithm returns projected scores that satisfy fairness while minimizing the divergence from the original scores.  We use KL divergence.
\end{itemize}
In-processing algorithms include:
\begin{itemize}
  \item \textit{Reductions~\textnormal{\citep{agarwal2018ReductionsApproachFair}}.}~~Limited to binary classification.  This algorithm returns a randomized ensemble of classifiers.
        The implementation from the AIF360 library is used~\citep{bellamy2018AIFairness360}, with $50$ iterations (by default).

  \item \textit{Adversarial~\textnormal{\citep{zhang2018MitigatingUnwantedBiases}}.}~~We use our own implementation of adversarial debiasing and align the group-conditional feature distributions in the penultimate layer of the network (instead of aligning the output logits).  For multiclass equal opportunity and equalized odds, we inject label information to the adversary in CDAN style~\citep{long2018ConditionalAdversarialDomain}.
\end{itemize}

\paragraph{Models.}

On tabular datasets (all except BiasBios), we evaluate the backbone models of logistic regression and gradient boosting decision tree~(GBDT), and additionally two-hidden-layer ReLU network~(MLP) on the ACSIncome dataset.  On the BiasBios text dataset, we fine-tune BERT models from the \texttt{bert-base-uncased} checkpoint~\citep{devlin2019BERTPretrainingDeep}.

For post-processing, predictors for the pointwise risk $\hat r$ and the group membership $\hat g$ are first trained without fairness constraints using one of the models above as the backbone, and are then post-processed to satisfy fairness.  To obtain the base predictors, we only (need to) train a single predictor $\hat f_{A,Y}:\calX\rightarrow\Delta^{\calA\times\calY}$ for the joint $(A,Y)$; this is because for the 0-1 loss and the fairness criteria that we consider (SP, EOpp, and EO), $\hat r$ and $\hat g$ can be expressed in terms of the predictor for $Y$, and the predictor for $A$ or the joint $(A,Y)$ depending on the criterion, all of which can be derived from $\hat f_{A,Y}$.  The same base predictor $\hat f_{A,Y}$ (with different backbones) is used to evaluate all post-processing algorithms.

For in-processing, predictors for $Y$ are trained directly with fairness constraints using the same backbone models and specialized optimization procedures.  Note that Adversarial is only applicable to neural models (MLP and BERT).

\paragraph{Calibrated Predictors for Post-Processing.}
Since labels are available in the post-processing split, we use them to calibrate the $\hat f_{A,Y}$ predictor used by post-processing algorithms prior to post-processing (although it should ideally be performed on a separate split).  This is for fair comparison of FairProjection and our algorithm to MBS: the former ones only use unlabeled data while MBS requires labeled data (both $A$ and $Y$).  This calibration step therefore makes all algorithms consume the same data resources.  The impact and benefits of calibration to our algorithm will be later examined in an ablation.

We use the calibration package from scikit-learn~\citep{pedregosa2011ScikitlearnMachineLearning}: on smaller datasets (Adult and COMPAS), Platt scaling is applied, otherwise isotonic regression.  Since both the pointwise risk and the group predictor are derived from $\hat f_{A,Y}$, the multicalibration condition in \cref{eq:cal.0} is satisfied if $\hat f_{A,Y}$ is \textit{distribution calibrated}~\citep{song2019DistributionCalibrationRegression}, i.e., $\Pr(A=a,Y=y\mid \hat f_{A,Y}=p)=p_{a,y}$ for all $p\in\Delta^{\calA\times\calY}$.

\begin{figure}[p]
  \centering
  \includegraphics[width=1\linewidth]{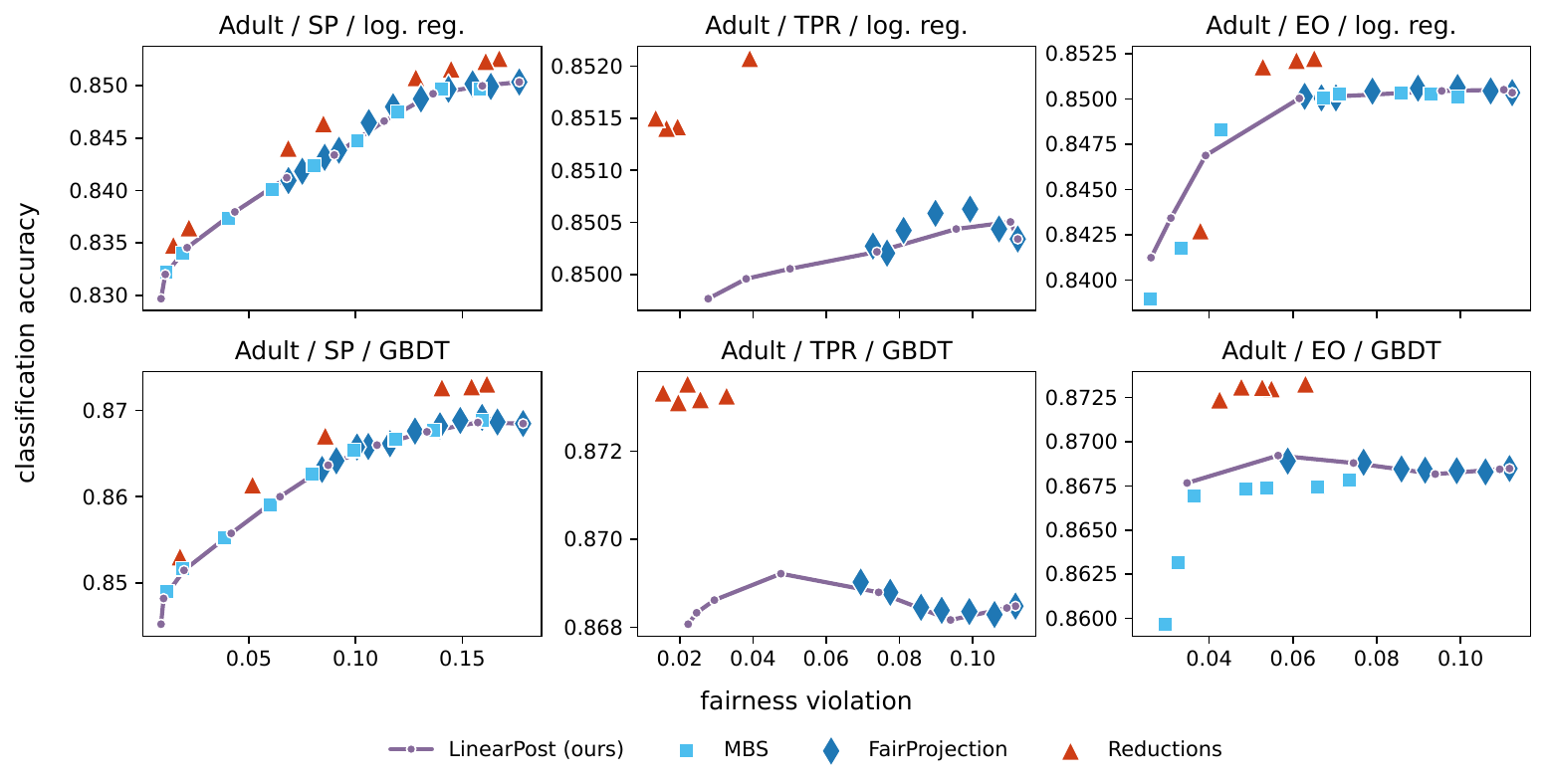}
  \caption{Tradeoffs between accuracy and fairness on Adult subject to SP, TPR (binary class equal opportunity), or EO, respectively. 
  Fairness violation is computed uniformly (see definitions in \cref{tab:equivalence}).  Average of five random seeds.}
  \label{fig:exp.adult}
\end{figure}

\begin{figure}[p]
  \centering
  \includegraphics[width=1\linewidth]{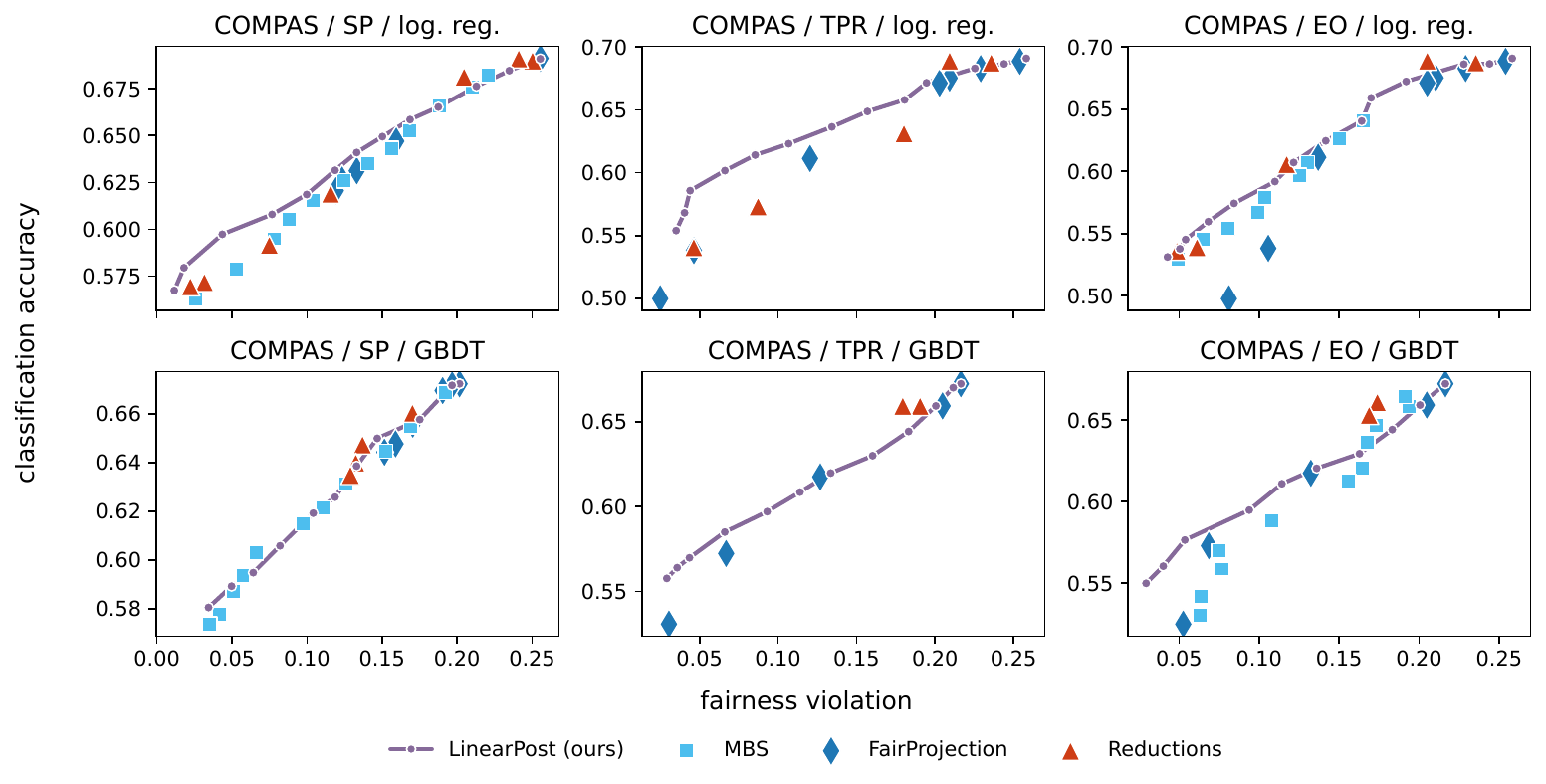}
  \caption{Tradeoffs between accuracy and fairness on COMPAS\@.  TPR is binary class equal opportunity.}
  \label{fig:exp.compas}
\end{figure}

\begin{figure}[p]
  \centering
  \includegraphics[width=1\linewidth]{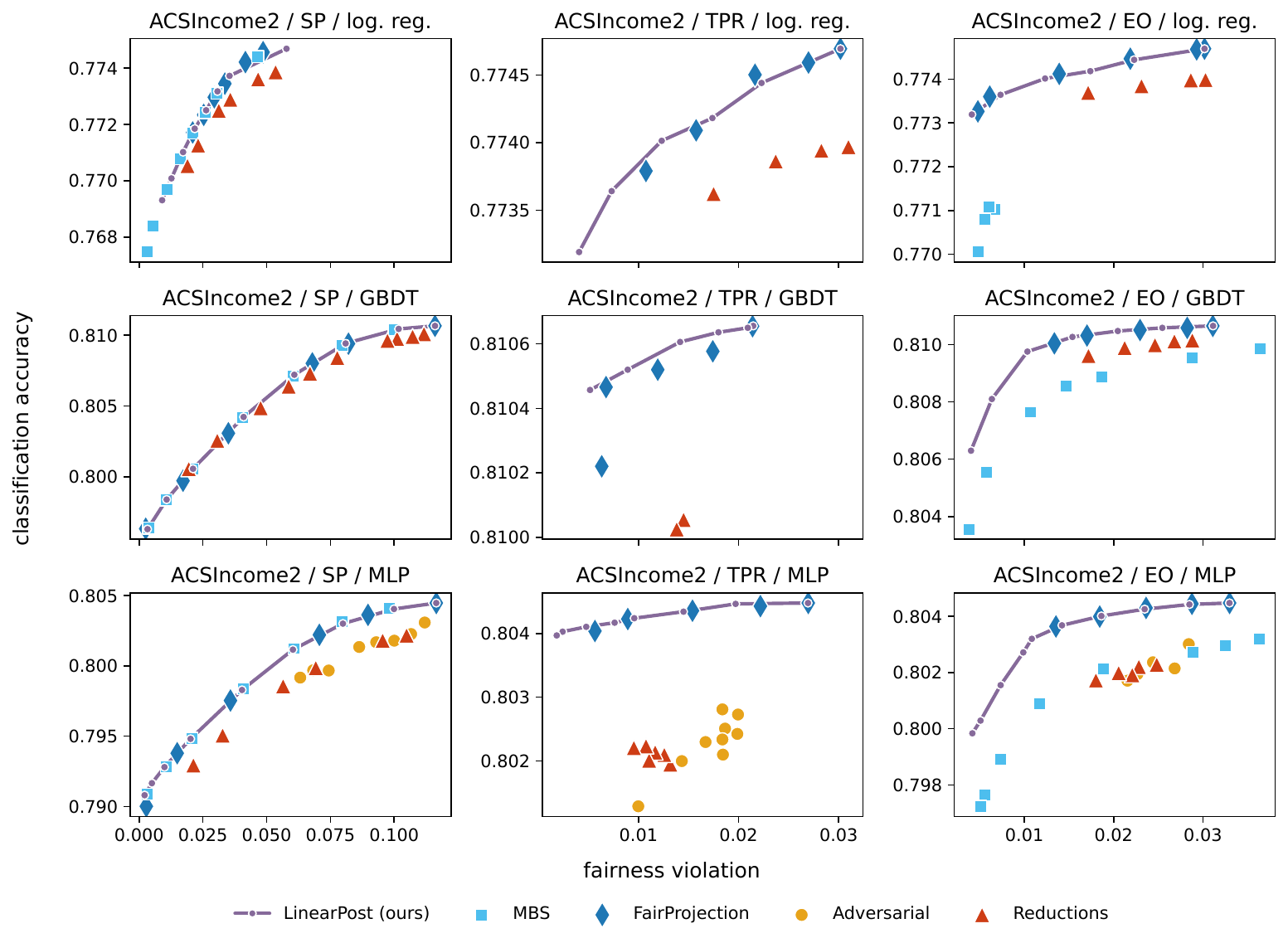}
  \caption{Tradeoffs between accuracy and fairness on ACSIncome2.  TPR is binary class equal opportunity.}
  \label{fig:exp.acsincome2}
\end{figure}

\begin{figure}[p]
  \centering
  \includegraphics[width=1\linewidth]{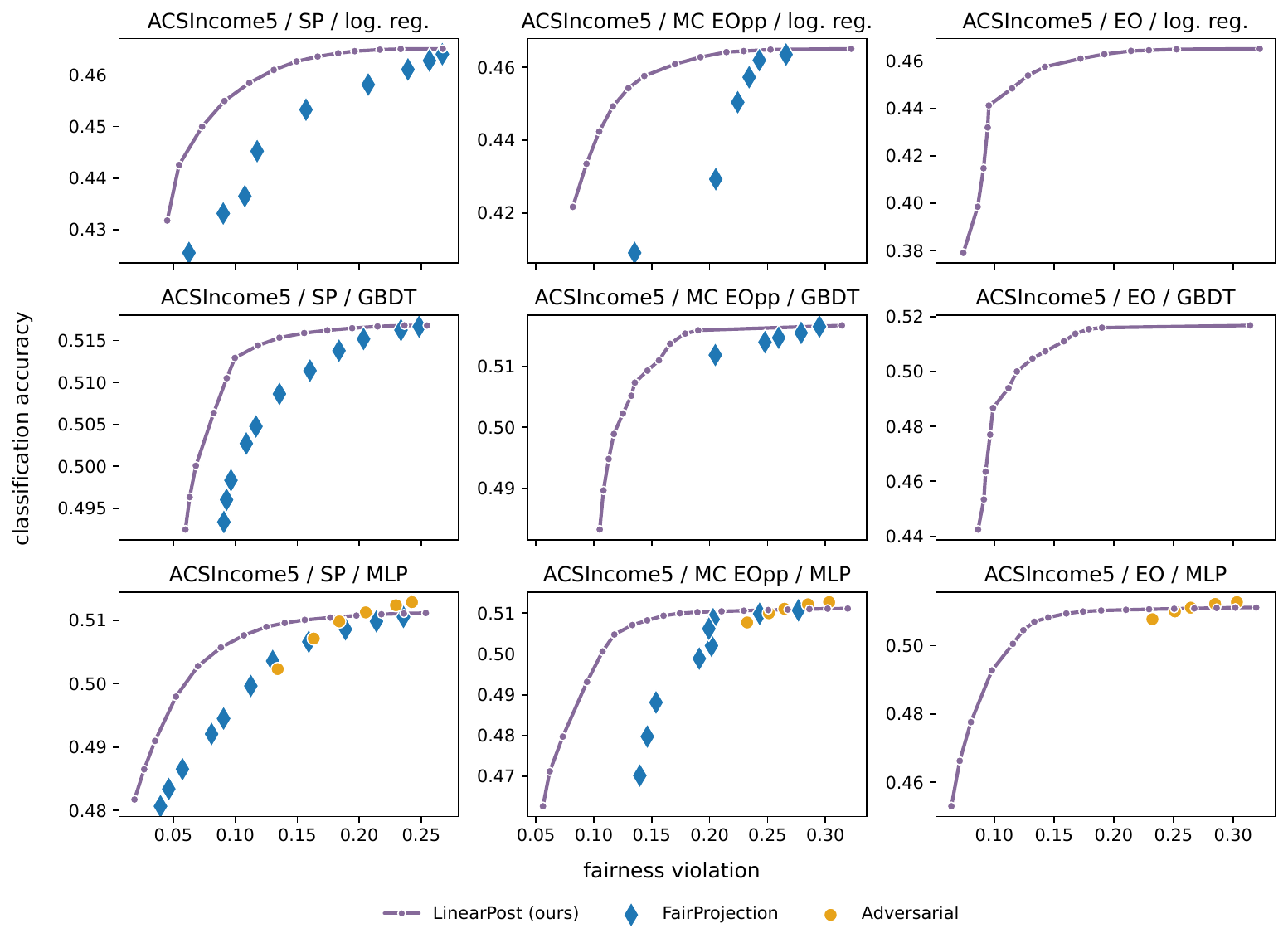}
  \caption{Tradeoffs between accuracy and fairness on ACSIncome5.  MC EOpp is multiclass equal opportunity.}
  \label{fig:exp.acsincome5}
\end{figure}

\begin{figure}[p]
  \centering
  \includegraphics[width=1\linewidth]{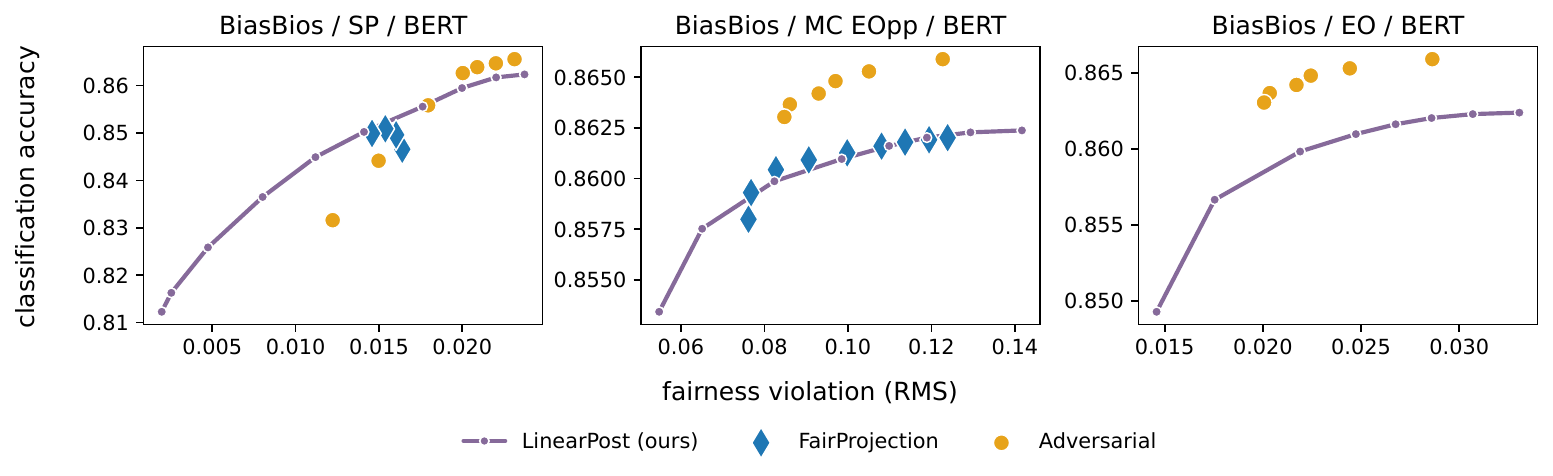}
  \caption{Tradeoffs between accuracy and fairness on BiasBios.  MC EOpp is multiclass equal opportunity.  Due to the large number of classes, to highlight incremental improvements, fairness violation is computed by taking the root mean square of the disparities (across $y\in\calY$ for SP and EOpp, and $y,y'\in\calY$ for EO). The RMS violation of MC EOpp is the same as the $\mathit{GAP}^{\textrm{TPR},\textrm{RMS}}$ metric in~\citep{ravfogel2020NullItOut}.}
  \label{fig:exp.biasbios}
\end{figure}


\subsection{Results}

\Cref{fig:exp.adult,fig:exp.compas,fig:exp.acsincome2,fig:exp.acsincome5,fig:exp.biasbios} plot the accuracy-fairness tradeoffs achieved by the fair algorithms to satisfy SP, EOpp, or EO, from sweeping the fairness tolerance until fairness stops improving.
All results are averaged over five random seeds for performing the post-processing/test split, model initialization, and the randomness of the learned classifier; standard deviation of our results can be found in our code repository\footref{fn:code}.
Furthermore, we plot the specified fairness tolerance $\alpha$ against the actual achieved fairness violation for our algorithm in \cref{fig:alpha}, as well as its post-processing running time in \cref{tab:runtime} (note that EO on BiasBios requires $2\sum_c\envert{\calI_c}=2\envert \calA \envert \calY^2=2\cdot 2\cdot 28^2$ constraints).

\paragraph{Comparison with Existing Algorithms.}
Among post-processing algorithms, ours algorithm achieves the best tradeoffs in most cases, especially on multiclass problems (ACSIncome5 and BiasBios).  Comparing to MBS, because its post-processing parameters are obtained by performing grid search on labeled data, it could achieve higher fairness but at the expense of significantly longer running time.   FairProjection encountered out-of-memory error when post-processing for EO on multiclass problems.

Between post-processing and in-processing, in-processing can attain higher accuracies (see discussions in \cref{sec:intro}), but in most case, particularly multiclass problems, it cannot achieve the same high levels of fairness as post-processing does.  The main difficulty with in-processing at achieving high fairness is likely due to optimization. Reductions requires an oracle for learning cost-sensitive classifiers and involves repeatedly training and combining a large number of them to converge to a fair classifier, so performance may be affected if the oracle is suboptimal, or if the number of iterations is limited due to computational budget (Reductions already has the longest running time in our experiments). Adversarial training suffers from instability due to the conflicting objective of the predictor (whose goal is risk minimization) and the adversary (fairness).  Indeed, ease of optimization is what distinguishes post-processing from in-processing: the ability to solve the post-processing problem to optimality enables post-processing to achieve better fairness.

\paragraph{High Fairness Regime.}

In our algorithm, when the fairness tolerance $\alpha$ is set to small values in hopes of achieving higher fairness, the improvement in fairness on the surrogate empirical problem $(\hat r,\hat g,\widehat\PP_X)$ may not necessarily transfer to the true population problem $(r,g,\PP_X)$ (see \cref{fig:alpha}).  Also, accuracy drops off more rapidly in the high fairness regime.

There are two main reasons for this.  The first is sample variance: achieving high fairness (i.e., small $\alpha$) on the empirical problem requires fitting the weights $\hat w$ to the training set more precisely, which would cause overfitting and therefore high variance; it could be alleviated with more samples.
The second is the error of the group predictor $\hat g$. As we have analyzed in \cref{thm:sensitivity}, exact fairness cannot be guaranteed if $\hat g$ is not optimal, and small $\alpha$ settings causes accuracy to be more sensitive to its suboptimality hence leading to larger drops in performance (of course, this rapid drop may simply be due to the inherent tradeoffs). Solutions in this case include training better group predictors or performing calibration: indeed, as we show in our ablation studies in \cref{fig:abl.adult,fig:abl.compas,fig:abl.acsincome2,fig:abl.acsincome5,fig:abl.biasbios} (discussed below), post-processing calibrated predictors yields fairer classifiers compared to their uncalibrated counterparts.



\paragraph{Calibration and Attribute-Awareness.}
In \cref{fig:exp.adult,fig:exp.compas,fig:exp.acsincome2,fig:exp.acsincome5,fig:exp.biasbios}, we evaluated the performance of our algorithm on post-processing calibrated $\hat f^{\mathrm{blind},\mathrm{cal}}_{A,Y}$ predictors in the attribute-blind setting. We perform two ablation to study the impact of calibration and attribute-awareness.

In the first set of ablations, we compare post-processing calibrated predictors to uncalibrated ones $\hat f^{\mathrm{blind},\mathrm{uncal}}_{A,Y}$ (in the attribute-blind setting), the results are presented in \cref{fig:abl.adult,fig:abl.compas,fig:abl.acsincome2,fig:abl.acsincome5,fig:abl.biasbios}.  We observe that in many cases, calibration leads to higher fairness in the resulting classifier, or better accuracy, or both, although the improvements are not empirically consistent across datasets, models, and fairness criteria.  Nonetheless, this partially corroborates our analysis in \cref{thm:calibration} that calibration should generally help with fairness.

In the second set of ablations, we explore the performance gap from attribute-awareness by training attribute-aware predictors $\hat f^{\mathrm{aware},\mathrm{uncal}}_{A,Y}$ and comparing to attribute-blind ones above (both uncalibrated), $\hat f^{\mathrm{blind},\mathrm{uncal}}_{A,Y}$.  The results are in the same figures.  Because attribute-aware predictors can leverage more information for prediction, we see that they generally have better accuracy than attribute-blind predictors, and at higher unfairness as expected.  In most cases, post-processing attribute-aware predictors achieves higher fairness because attribute-awareness provides better group predictions.

\section{Conclusion and Limitations}\label{sec:conclusion}

Based on a representation result of the Bayes optimal fair classifier, in this paper, we present a post-processing fair classification learning algorithm that covers a variety of group fairness criteria (including SP, EOpp, and EO), and applicable to binary and multiclass problems in both attribute-aware and blind settings.

In addition to the limitations regarding randomization (although inherent to group fairness in classification) mentioned in \cref{rem:uniqueness.violation}, and intersecting groups in \cref{sec:ins}, we discuss another limitation/future work below regarding threshold invariance.

\paragraph{Threshold-Invariance.}

Some practical scenarios may occasionally require adjusting the threshold of the classifier after deployment, e.g., to change the acceptance rate in response to a \textit{label shift} in the underlying distribution~\citep{lipton2018DetectingCorrectingLabel}.
However, in our framework, fairness is not preserved under changes to the threshold of our post-processed scores \cref{eq:0}; changing the acceptance rate while maintaining fairness requires redoing post-processing under a class weighting.

Optimal and threshold-invariant fair scores can be obtained via fair representation learning~\citep{zemel2013LearningFairRepresentations}, which is an in-processing algorithm.  By treating it as a fair regression problem, threshold-invariant fair scores could also be obtained using post-processing algorithms for fair regression~\citep{chzhen2020FairRegressionWasserstein}, but their objective (e.g., minimizing MSE) may not be aligned with performance metrics such as classification error.  It remains an open problem on how to obtain optimal and threshold-invariant fair scores via post-processing---optimality here is in the sense that they are capable of exploring the Pareto front of the TPRs of each class from varying the threshold.

\section*{Acknowledgements}

RX would like to thank Haoxiang Wang for setting up an Amazon EC2 instance.  HZ is partially supported by a research grant from the Amazon-Illinois Center on AI for Interactive Conversational Experiences~(AICE) and a Google Research Scholar Award.


\bibliographystyle{plainnat-eprint}
\bibliography{references}

\newpage
\appendix

\section{Omitted Proofs for Section~\ref{sec:main}}\label{sec:proof}

\begin{proof}[Derivation of \hyperlink{eq:dual}{Dual LP}]
  The Lagrangian of \LPp\ is
  \begin{align}
    L(\pi,q,\phi,\psi^+,\psi^-) & = \begin{multlined}[t]
    \E_{X}\!\Bigg[ \sum_{y\in\calY}   r(X,y)\pi(X,y) \Bigg]
                                + \E_{X}\! \Bigg[ \Bigg( 1 - \sum_{y\in\calY}\pi(X,y) \Bigg)\phi(X) \Bigg] \\
                                + \sum_{c\in[C]}\sum_{k\in\calI_c}\rbr*{ -\frac\alpha2 + q_c - \E_{X}\sbr*{ \frac{g(X,k)}{\Pr(Z_k=1)}\pi(X,y_c)} } \psi^+_{c,k} \\
                                + \sum_{c\in[C]}\sum_{k\in\calI_c}\rbr*{ -\frac\alpha2 - q_c + \E_{X}\sbr*{ \frac{g(X,k)}{\Pr(Z_k=1)}\pi(X,y_c)} } \psi^-_{c,k}.\end{multlined}
  \end{align}
  Collecting terms, we get
  \begin{align}
    \MoveEqLeft L(\pi,q,\phi,\psi^+,\psi^-)                                                                                                            \\
     & = \begin{multlined}[t]
           \E_{X}\sbr{\phi(x)} +  \sum_{c\in[C]}\sum_{k\in\calI_c}\rbr*{q_c (\psi^+_{c,k}-\psi^-_{c,k}) - \frac\alpha2(\psi^+_{c,k} + \psi^-_{c,k}) } \\
           + \E_{X}\!\Bigg[ \sum_{y\in\calY}\Bigg(r(X,y) -  \Bigg(\underbrace{\phi(X) +  \sum_{c\in [C]}\1[y_c=y]\sum_{k\in\calI_c} g(X,k)\frac{\psi^+_{c,k}-\psi^-_{c,k}}{\Pr(Z_k=1)}}_{(\star)}\Bigg)   \Bigg)\pi(X,y) \Bigg].
         \end{multlined}
  \end{align}
  By strong duality, $\min_{\pi\geq0,q}\max_{\phi,\psi^+\geq0,\psi^-\geq0} L=\max_{\phi,\psi^+\geq0,\psi^-\geq0}\min_{\pi\geq0,q} L$.  If $r(x,y) < (\star)$ for some $(x,y)$, then we can send $L$ to $-\infty$ by setting $\pi(x,y)=\infty$, so we must have that $r(x,y)\geq(\star)$ for all $x,y$. But with this constraint, the best we can do for $\min_{\pi\geq0,q} L$ is to set $\pi=0$, so the last line is omitted.  Similarly, we must have $\sum_{k\in\calI_c}(\psi^+_{c,k}-\psi^-_{c,k})=0$ from its interaction with $q_c$.
\end{proof}

Next, we prove \cref{rem:continuity} by showing that random perturbation makes the joint distribution of $(r(X)+\xi,g(X))$ continuous (stated in \cref{ass:continuity} below, which has appeared in prior work with similar forms)---used as input features by the post-processing, which in turn implies \cref{ass:uniqueness}.

\begin{assumption}[Continuity]\label{ass:continuity}
  Given $r:\calX\rightarrow\RR^\calY$ and $g:\calX\rightarrow\RR^K$, the push-forward distribution of $(r,g)\sharp\PP_X$ supported on $\RR^{\calY\times K}$ does not give mass to any strict linear subspace that has a non-zero component in the $\RR^\calY$ subspace.
\end{assumption}

\begin{lemma}\label{lem:continuity}
  If $(r,g)$ satisfies \cref{ass:continuity}, then it also satisfies \cref{ass:uniqueness}.
\end{lemma}

\begin{proof}
  Observe that $\argmin_{y} \rbr{r(X,y) - \sum_{k} g(X,k)w(y,k)}$ is a linear classifier on the features $(r(x),g(x))\in \RR^{\calY\times K}$ with coefficients $\beta_{y}\coloneqq(\be_{y}, w(y,\cdot))$ for each output class $y$, and the coefficients always have a non-zero component in the $\calY$-coordinates.  On the other hand, the set of points $x$ where $\envert{\argmin_{y} \rbr{r(x,y) - \sum_{k} g(x,k)w(y,k)}}>1$ is contained in the set whose features lie on the boundary between two (or more) classes, i.e., $\exists y,y'$, $y\neq y'$ s.t.\ $\beta_y^\top (r(x),g(x)) = \beta_{y'}^\top (r(x),g(x))$. \Cref{ass:continuity} simply states that this set has measure zero, hence satisfying \cref{ass:uniqueness}.
\end{proof}

\begin{proof}[Proof of \cref{rem:continuity}]
  Since the linear subspaces considered in \cref{ass:continuity} are strict, they can be represented by $\{(u,v): u\in\RR^\calY, v\in\RR^K,  a^\T u + b^\T v=0\} $ for some $a\in\RR^\calY$, $a\neq 0$, and $b\in\RR^K$.  Then
  \begin{align}
    \MoveEqLeft(r+\xi,g)\sharp\PP_X\rbr*{\{(u,v): u\in\RR^\calY, v\in\RR^K,  a^\T u + b^\T v=0\}}\\
     & = \Pr_{X,\xi}\rbr*{ a^\T r(X) +a^\T\xi +b^\T g(X) =0} \\
     & = \E_{X}\sbr*{\Pr_{\xi}\rbr*{ a^\T\xi = - a^\T r(X)-b^\T g(X) } } \\
     & = 0
  \end{align}
  because $a^\T \xi \neq 0$ has a continuous distribution independent of $X$.
\end{proof}

\section{Proofs for Section~\ref{sec:sensitivity}}\label{sec:proof.sensitivity}

Note the fact that $0$-group fairness is satisfied by (randomized) classifiers whose output does not depend on the input, such as the constant classifier:

\begin{proposition}\label{prop:constant}
  Let $\pi(x,\cdot)=p$ for some fixed $p\in\Delta^{\calY}$, $\forall x\in\calX$.  Then for any fairness constraints $\calC$ and group predictor $g:\calX\rightarrow[0,1]^K$,
    \begin{equation}
      \max_{k,k'\in\calI_c}\envert*{\E_X \sbr*{\rbr*{\frac{g(X,k)}{\E_X[g(X,k)]}-\frac{g(X,k')}{\E_X[g(X,k')]}} \pi(X, y_c)  }} =0,\quad\forall c\in[C]. \label{eq:constant}
    \end{equation}
\end{proposition}

\begin{proof}
  For all $\forall c\in[C]$ and $k\in\calI_c$,
  \begin{align}
    \E_X \sbr*{\frac{g(X,k)}{\E_X[g(X,k)]} \pi(X,y_c)}
    &= \E_X \sbr*{\frac{g(X,k)}{\E_X[g(X,k)]}} p_{y_c} = p_{y_c},
  \end{align}
  so the difference in \cref{eq:constant} is zero.
\end{proof}

\begin{proof}[Proof of \cref{thm:sensitivity}]
  For clarity, we denote
  \begin{equation}
    \abr{r,\pi} =   \E_X\!\Bigg[ \sum_{y\in\calY}r(X,y)\pi(X,y)\Bigg],\quad   \abr{\gamma_k,\pi_y} =  \E_X\sbr*{\frac{g(X,k)}{\E_X[g(X,k)]} \pi(X, y)},
  \end{equation}
  where
  \begin{equation}
  \gamma_k(x)=\frac{g(x,k)}{\E_X[g(X,k)]},\quad \pi_y(x)=\pi(x,y).
  \end{equation}

  \paragraph{Part 1 \textnormal{(Deviation of $\hat r$ from $r$)}.}
  Because $\pi,\hat\pi$ are both feasible solutions to $\LP(r,g,\PP_X,\calC,\alpha)$ and $\LP(\hat r,g,\PP_X,\calC,\alpha)$, and are minimizers of the respective problem, we have
  \begin{align}
    0
     & \leq \abr{r,\hat\pi} - \abr{r,\pi}                                 \\
     & = \abr{r-\hat r,\hat\pi} + \abr{\hat r,\hat\pi-\pi} + \abr{\hat r-r,\pi} \\
     & \leq \abr{r-\hat r,\hat\pi}  + \abr{\hat r-r,\pi}                    \\
     & = \abr{r-\hat r,\hat\pi-\pi}                                           \\
     & \leq  \E_X\!\Bigg[ \sum_{y\in\calY} \envert{r(X,y)-\hat r(X,y)}\Bigg]              \\
     & = \varepsilon_{\hat r}
  \end{align}
  by H\"older's inequality, since $\max_y\envert{\hat\pi(x,y)-\pi(x,y)}\leq 1$.  Note that $\abr{r,\pi}=\mathrm{OPT}$.

  \paragraph{Part 2 \textnormal{(Deviation of $\hat g$ from $g$)}.}
  We start with the bound on constraint violation.
  For all $c\in[C]$,
  \begin{align}
    \MoveEqLeft \max_{k,k'\in\calI_c} \envert{ \abr{\gamma_k-\gamma_{k'},\hat\pi_{y_c}}}                                                                                                                          \\
     & = \max_{k,k'\in\calI_c} \envert{ \abr{\hat\gamma_k-\hat\gamma_{k'},\hat\pi_{y_c}} + \abr{\gamma_k - \hat\gamma_k+\hat\gamma_{k'}-\gamma_{k'},\hat\pi_{y_c}}}                                                              \\
     & \leq \max_{k,k'\in\calI_c} \envert{ \abr{\hat\gamma_k-\hat\gamma_{k'},\hat\pi_{y_c}}} +  2\max_{k\in\calI_c} \envert{ \abr{\gamma_k - \hat\gamma_k,\hat\pi_{y_c}}}                                                     \\
     & \leq \alpha +  2\max_{k\in\calI_c} \envert{ \abr{\gamma_k - \hat\gamma_k,\hat\pi_{y_c}}}                                                                  &  & \textrm{(second constraint of \LPp)}           \\
     & \leq \alpha +  2\max_{k\in\calI_c}\max_{\pi'} \envert{ \abr{\gamma_k - \hat\gamma_k,\pi'_{y_c}}}                                                                                                           \\
     & = \alpha +  2\max_{k\in\calI_c} \envert{ \abr{(\gamma_k - \hat\gamma_k)_+,\mathbf 1}}                                                                  &  & \textrm{($\hat\pi_{y_c}(x)\in[0,1]$ for all $x$)} \\
     & = \alpha + \max_{k\in\calI_c} \abr{ \envert{\gamma_k - \hat\gamma_k},\mathbf 1}                                                                                                                            \\
     & \leq \alpha + \varepsilon_{\hat g},
  \end{align}
  where $(a)_+\coloneqq\max(0,a)$ element-wise. The second last equality uses the fact that $\abr{a,\mathbf 1}=0\implies \envert{\abr{a,\mathbf 1}}=2\abr{(a)_+,\mathbf 1}$, and we verify that $\abr{\gamma_k-\hat\gamma_k, \mathbf 1}=0$: by  \cref{eq:constraint.fn}, $\abr{\gamma_k, \mathbf 1}=\int_\calX\gamma(x,k)\dif\PP_X(x)=\int_\calX\Pr(Z_k=1,X=x)\dif x/\Pr(Z_k=1) = 1$.  In additionally, constraint violation cannot exceed $1$, because
  \begin{align}
    \max_{k,k'\in\calI_c} \envert{ \abr{\gamma_k-\gamma_{k'},\hat\pi_{y_c}}}
     & \leq \max_{k,k'\in\calI_c}\max_{\pi'} \envert{ \abr{\gamma_k-\gamma_{k'},\pi'_{y_c}}} \\
     & = \frac12\max_{k,k'\in\calI_c} \envert{ \abr{\gamma_k-\gamma_{k'},\mathbf 1}}         \\
     & \leq \max_{k\in\calI_c} \envert{ \abr{\gamma_k,\mathbf 1}}                            \\
     & =1,
  \end{align}
  so
  \begin{equation}
    \max_{k,k'\in\calI_c} \envert{ \abr{\gamma_k-\gamma_{k'},\hat\pi_{y_c}}} \leq \min(1,\alpha+\varepsilon_{\hat g}),\quad\forall c\in[C],
  \end{equation}
  concluding the bound on constraint violation.

  By symmetry, we can also show that 
  \begin{equation}\label{eq:sensitivity.proof.33}
    \max_{k,k'\in\calI_c} \envert{ \abr{\hat\gamma_k-\hat\gamma_{k'},\pi_{y_c}}} \leq \min(1,\alpha+\varepsilon_{\hat g}),\quad\forall c\in[C],
  \end{equation}
  which is the constraint violation of $\pi$ on the plugin problem $\hat\gamma$;  this fact will be used below.

  For the bound on the excess risk, we begin by constructing a solution $\tilde\pi$ in terms of $\pi$ that is feasible on the plugin problem $\LPp(r,\hat g,\PP_X,\calC,\alpha)$:
  \begin{equation}
    \tilde\pi(x,\cdot) = (1-\beta)\pi(x,\cdot) + \beta \be_1,\quad\forall x\in\calX,
  \end{equation}
  for some $\beta\in[0,1]$ to be specified---with probability $1-\beta$, $\tilde\pi$ defers to the original classifier $\pi$, and w.p.\ $\beta$ always output class $1$.
  Because the constant classifier, namely the $\beta\be_1$ component of $\tilde\pi$, satisfies $0$-group fairness (\cref{prop:constant}), by \cref{eq:sensitivity.proof.33}, the constraint violation of $\tilde\pi$ on $\LP(r,\hat g,\PP_X,\calC,\alpha)$ is
  \begin{align}
    \MoveEqLeft\max_{k,k'\in\calI_c} \envert{ \abr{\hat\gamma_k-\hat\gamma_{k'},\tilde\pi_{y_c}}} = (1-\beta) \max_{k,k'\in\calI_c} \envert{  \abr{\hat\gamma_k-\hat\gamma_{k'},\pi_{y_c}}} \leq (1-\beta)\min(1,\alpha+\varepsilon_{\hat g}).
  \end{align}
  For $\tilde\pi$ to be feasible, the right-hand side cannot exceed $\alpha$, so we set
  \begin{equation}
    \beta\geq\min\rbr*{1-\alpha, \frac{\varepsilon_{\hat g}}{\alpha+\varepsilon_{\hat g}}}.
  \end{equation}
  Then, because $\hat\pi$ is a minimizer of $\LPp(r,\hat g,\PP_X,\calC,\alpha)$, and $\tilde\pi$ a feasible solution,
  \begin{align}
    \abr{r,\hat\pi-\pi}
    &\leq\abr{r,\tilde\pi-\pi} \\
    &\leq \enVert r_\infty \E_X \!\Bigg[\sum_{y\in\calY}  \envert{\tilde\pi(X,y)-\pi(X,y)}\Bigg] \\
    &\leq 2\enVert r_\infty \beta\\
    & \leq2 \enVert r_\infty \min\rbr*{1-\alpha, \frac{\varepsilon_{\hat g}}{\alpha+\varepsilon_{\hat g}}},
  \end{align}
  concluding the bound on the excess risk.
\end{proof}

\begin{proof}[Proof of \cref{thm:calibration}]
  Let $\hat g^\mathrm{cal}$ denote the multicalibrated version of $\hat g$ to the level sets of $(r,\hat g)$, that is, for all $x\in\calX$,
  \begin{equation}
    \hat g^\mathrm{cal}(x,k) = \E\sbr*{g(X,k) \mid r(X)=u, \hat g(X)=v},\quad\text{where } u=r(x),\, v = \hat g(x),
  \end{equation}
  then it follows that for all $k$, $\E_X[\hat g^\mathrm{cal}(X,k)\1[X\in S]] = \E_X[g(X,k)\1[X\in S]]$, $\forall S\in\calS(r,\hat g)$, and $\E_X[\hat g^\mathrm{cal}(X,k)] = \E_X[g(X,k)]$. So,
  \begin{align}
    \varepsilon^\mathrm{cal}_{\hat g} 
    &=\max_{k\in[K]}\sum_{S\in\calS(r,\hat g)} \envert*{ \E_X\sbr*{\rbr*{  \frac{\hat g(X,k)}{\E_X[\hat g(X,k)]}-\frac{\hat g^\mathrm{cal}(X,k)}{\E_X[\hat g^\mathrm{cal}(X,k)]}  } \1[X\in S] }}.
  \end{align}

  With the same notation and analysis for the bound on constraint violation in \cref{thm:sensitivity} Part~2, for all $c\in[C]$,
  \begin{align}
    \MoveEqLeft\max_{k,k'\in\calI_c} \envert{ \abr{\gamma_k-\gamma_{k'},\tilde\pi_{y_c}}}                                                                  \\
     & \leq \alpha  + \varepsilon^\mathrm{cal}_{\hat g}   + 2\max_{k\in\calI_c}\envert*{\E_X \sbr*{\rbr*{\frac{\hat g^\mathrm{cal}(X,k)}{\E_X[\hat g^\mathrm{cal}(X,k)]}-\frac{g(X,k)}{\E_X[g(X,k)]}} \tilde\pi(X, y_c)  }} \\
     & = \alpha  + \varepsilon^\mathrm{cal}_{\hat g}   + \frac{2}{\E_X[g(X,k)]} \max_{k\in\calI_c}\envert*{\E_X \sbr*{\rbr*{\hat g^\mathrm{cal}(X,k)-g(X,k)} \tilde\pi(X, y_c)  }} \\
     & = \alpha  + \varepsilon^\mathrm{cal}_{\hat g}   + \frac{2}{\E_X[g(X,k)]} \max_{k\in\calI_c}\envert*{\E \sbr*{ \E_X \sbr*{\rbr*{\hat g^\mathrm{cal}(X,k)-g(X,k)} \tilde\pi(X, y_c) \mid r(X)=u, \hat g(X)=v }}}\\
     & \propto \alpha  + \varepsilon^\mathrm{cal}_{\hat g}   + \frac{2}{\E_X[g(X,k)]} \max_{k\in\calI_c}\envert*{\E \sbr*{ \E_X \sbr*{\rbr*{\hat g^\mathrm{cal}(X,k)-g(X,k)}  \mid r(X)=u, \hat g(X)=v }}}\\
     &= \alpha  + \varepsilon^\mathrm{cal}_{\hat g},
  \end{align}
  where line 3 and 6 are by the multicalibration condition, and line 5 is because by \cref{eq:same}, $\tilde\pi(x,y_c)$ is constant for all $x$ that map to the same features.
\end{proof}

  \begin{figure}[t]
    \centering
    \includegraphics[width=0.53\linewidth]{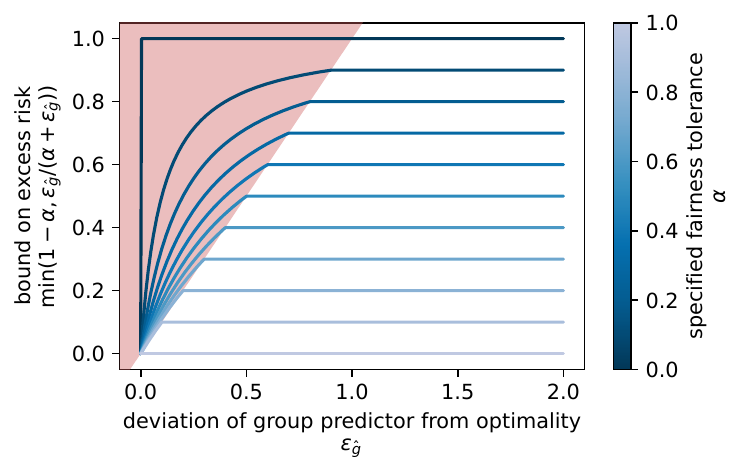}
    \caption{The upper bound on the excess risk in \cref{thm:sensitivity} Part~\ref{thm:sensitivity.2} as a function of the deviation of $\hat g\neq g$ defined in \cref{eq:err.g}, plotted for different tolerance settings $\alpha\in[0,1]$.  \Cref{ex:tightness} construct problem instances that match the upper bound up to a constant factor for the shaded region.}
    \label{fig:constraint.bound}
  \end{figure}

We illustrate our bound in \cref{thm:sensitivity} Part~\ref{thm:sensitivity.2} on the excess risk of the classifier post-processed using an inaccurate group predictor $\hat g\neq g$ in \cref{fig:constraint.bound}.  Each line corresponds to a different fairness tolerance setting $\alpha$, and the lines plot the excess risk bound as a function of the deviation $\varepsilon_{\hat g}$, defined in \cref{eq:err.g}, of $\hat g$ from the optimal $g$.  In what follows, we construct problem instances that match the upper bound up to a constant factor (the constructions cover the shaded region in \cref{fig:constraint.bound}).

\begin{example}\label{ex:tightness}
We construct binary classification problems in the attribute-blind setting, with classification error (0-1 loss) as the objective and statistical parity as the fairness criterion.  The problems are constructed such that the optimal unconstraint classifier is already fair (i.e., with respect to the optimal group predictor $g$, $\LP(r,g,\PP_X,\calC,\alpha)$), but not on the plugin problem $\LP(r,\hat g,\PP_X,\calC,\alpha)$ w.r.t.\ the inaccurate group predictor $\hat g$, so the excess risk is incurred from satisfying fairness on the plugin problem.

We define the data distribution as follows:
  \begin{itemize}
    \item The sensitive attribute is binary, $A\in\{0,1\}$, and has a uniform marginal distribution, $\Pr(A=0)=\Pr(A=1)=1/2$.
    \item The input is supported on two points, $\calX=\{0,1\}$, with $\Pr(X=0\mid A=1)=1-\Pr(X=0\mid A=0)$ and $\Pr(X=0\mid A=0)\in[0,1/2]$ varying across different instances.

      $X$ also has a uniform marginal distribution, because $\Pr(X=0)=\sum_{a\in\{0,1\}}\Pr(X=0\mid A=a)\Pr(A=a)=1/2=\Pr(X=1)$.

    \item The label $Y=X$.
  \end{itemize}

Denote $\pi_x=\pi(x,1)$, then the risk (classification error) of $\pi$ is
\begin{equation}
   R(\pi) = \frac12 \pi_0 + \frac12(1-\pi_1)  = \frac12+\frac12(\pi_0-\pi_1),\label{eq:ex.0}
\end{equation}
%
  and statistical parity requires that 
  \begin{align}
    \alpha & \geq  
           \begin{multlined}[t]\bigg|\frac{\pi_0}2\rbr*{\frac{\Pr(A=0\mid X=0)}{\Pr(A=0)} - \frac{\Pr(A=1\mid X=0)}{\Pr(A=1)}} \\
           + \frac{\pi_1}2\rbr*{\frac{\Pr(A=0\mid X=1)}{\Pr(A=0)} - \frac{\Pr(A=1\mid X=1)}{\Pr(A=1)}} \bigg| \end{multlined}\\
           & =\envert*{ (\Pr(X=0\mid A=0)-\Pr(X=0\mid A=1)) (\pi_0-\pi_1)}                                                                                                                                      \\
           & = \envert*{(1-2\Pr(X=0\mid A=0))(\pi_0-\pi_1)}                                                                                                                                                     \\
           & \eqqcolon \delta\envert*{\pi_0-\pi_1} \label{eq:ex.1}
  \end{align}
  by construction of the data distribution, where we defined $\delta=1-2\Pr(X=0\mid A=0)\in[0,1/2]$ whose value depends on the choice of $\Pr(X=0\mid A=0)\in[0,1/2]$ when creating the problem instance.

Combining \cref{eq:ex.0,eq:ex.1}, the true problem $\LP(r,g,\PP_X,\calC,\alpha)$ can be written as
  \begin{equation}
    \min_{\pi_0,\pi_1\in[0,1]}\frac12+\frac12(\pi_0-\pi_1) \quad \textrm{subject to}\quad \envert*{\pi_0-\pi_1} \leq \frac\alpha\delta,
  \end{equation}
and the optimal value is $1/2\cdot \max(1-\alpha/\delta,0)$, attained by $\pi_0=0$ and $\pi_1=\min(\alpha/\delta,1)$ with the convention $0/0\coloneqq\infty$.

Now, we let the inaccurate plugin group predictor $\hat g$ satisfy $\widehat\Pr(A=0)=\widehat\Pr(A=1)$, and $\widehat\Pr(X=0\mid A=1)=1-\widehat \Pr(X=0\mid A=0)$ with $\widehat\Pr(X=0\mid A=0)\in[0,\ \Pr(X=0\mid A=0)]$ depending on the problem instance, then
\begin{align}
  \varepsilon_{\hat g} & = \frac12\enVert*{\frac{\widehat\Pr(A=0\mid X=\cdot)}{\widehat\Pr(A=0)} - \frac{\Pr(A=0\mid X=\cdot)}{\Pr(A=0)}}_1 \\
                & = \enVert*{\widehat\Pr(X=\cdot\mid A=0) - \Pr(X=\cdot\mid A=0)}_1                                          \\
                & = 2\envert*{\widehat\Pr(X=0\mid A=0) - \Pr(X=0\mid A=0)}                                                   \\
                & \in [0,1-\delta].
\end{align}

The plugin problem $\LP(r,\hat g,\PP_X,\calC,\alpha)$ is, similarly,
  \begin{equation}
    \min_{\hat\pi_0,\hat\pi_1\in[0,1]}\frac12+\frac12(\hat\pi_0-\hat\pi_1) \quad \textrm{subject to}\quad \envert*{\hat\pi_0-\hat\pi_1} \leq \frac\alpha{\hat\delta},
  \end{equation}
where the optimal value is $1/2\cdot \max(1-\alpha/{\hat\delta},0)$, and
  \begin{align}
    \hat\delta =  1-2\widehat\Pr(X=0\mid A=0) = 1-2\rbr*{\Pr(X=0\mid A=0) - \frac{\varepsilon_{\hat g}}2} = \delta + \varepsilon_{\hat g}.
  \end{align}

  Therefore, we can create problem instances with fairness tolerances $\alpha\in[0,1]$ and plugin group predictors with deviations $\varepsilon_{\hat g}\in[0,1-\alpha]$ (by choosing $\delta=\alpha$)---covering the shaded region in \cref{fig:constraint.bound}---where the excess risk of classifiers post-processed from solving the plugin problem $\LP(r,\hat g,\PP_X,\calC,\alpha)$ is
  \begin{equation}
    \frac12\max\rbr*{1-\frac\alpha{\delta+\varepsilon_{\hat g}},0} - \frac12\max\rbr*{1-\frac\alpha\delta,0} =\frac12\rbr*{1-\frac\alpha{\alpha+\varepsilon_{\hat g}}} =\frac12\rbr*{\frac{\varepsilon_{\hat g}}{\alpha+\varepsilon_{\hat g}}} \leq \frac12 (1-\alpha),
  \end{equation}
  which matches the upper bound in \cref{thm:sensitivity} Part~\ref{thm:sensitivity.2} by a factor of $1/4$ (because $\|r\|_\infty = 1$ in our construction).
\end{example}

\section{Proofs for Section~\ref{sec:finite.sample}}\label{sec:proof.sample}

To analyze the sample complexity in \cref{thm:sample} for estimating the parameters $\hat w$ of the post-processing for deriving the classifier $\hat h$, we proceed by analyzing the complexity of the hypothesis class of post-processings---which are linear multiclass classifier on the $(\calY\times K)$-dimensional features computed by $(r,g)$---and then applying uniform convergence bounds.

We recall the definitions of VC dimension, pseudo-dimension, and relevant uniform convergence bounds.

\begin{definition}[Shattering]
  Let $\calH$ be a class of binary functions from $\calX$ to $\{0,1\}$. A set $\{x_1,\dots,x_N\}\subseteq\calX$ is said to be shattered by $\calH$ if $\forall b_1,\dots,b_N\in\{0,1\}$ binary labels, $\exists h\in\calH$ s.t.\ $h(x_i)=b_i$ for all $i\in[N]$.
\end{definition}

\begin{definition}[VC Dimension]
  Let $\calH$ be a class of binary functions from $\calX$ to $\{0,1\}$. The VC dimension of $\calH$, denoted by $\dVC(\calH)$, is the size of the largest subset of $\calX$ shattered by $\calH$.
\end{definition}

\begin{definition}[Pseudo-Shattering]
  Let $\calF$ be a class of functions from $\calX$ to $\RR$. A set $\{x_1,\dots,x_N\}\subseteq\calX$ is said to be pseudo-shattered by $\calF$ if $\exists t_1,\dots,t_N\in\RR$ threshold values s.t.\ $\forall b_1,\dots,b_N\in\{0,1\}$ binary labels, $\exists f\in\calF$ s.t.\ $\1[f(x_i)\geq t_i] = b_i$ for all $i\in[N]$.
\end{definition}

\begin{definition}[Pseudo-Dimension]
  Let $\calF$ be a class of functions from $\calX$ to $\RR$. The pseudo-dimension of $\calF$, denoted by $\dP(\calF)$, is the size of the largest subset of $\calX$ pseudo-shattered by $\calF$.
\end{definition}

\begin{theorem}[Pseudo-Dimension Uniform Convergence]\label{thm:pd}
  Let $\calF$ be a class of functions from $\calX$ to $[0,M]$, and i.i.d.\ samples $x_1,\dots,x_N\sim\PP_X$.  Then with probability at least $1-\delta$ over the samples, $\forall f\in\calF$,
  \begin{equation}
    \envert*{ \E f(X) - \frac1N\sum_{i=1}^N f(x_i) } \leq cM\sqrt{\frac{\dP(\calF)+\ln1/\delta}{N}}
  \end{equation}
  for some universal constant $c$.
\end{theorem}

This can be proved via a reduction to the VC uniform convergence bound, see, e.g., Theorem~6.8 of \citep{shalev-shwartz2014UnderstandingMachineLearning} and Theorem~11.8 of \citep{mohri2018FoundationsMachineLearning}.  We will use this theorem to establish the following uniform convergence result for weighted loss of binary functions:

\begin{theorem}\label{thm:im.vc}
  Let $\calH$ be a class of binary functions from $\calX$ to $\{0,1\}$, $w:\calX\rightarrow[0,M]$ a weight function, and i.i.d.\ samples $x_1,\dots,x_N\sim\PP_X$.  Then with probability at least $1-\delta$ over the samples, $\forall h\in\calH$,
  \begin{equation}
    \envert*{ \E w(X)h(X) - \frac1N\sum_{i=1}^N w(x_i)h(x_i) } \leq cM\sqrt{\frac{\dVC(\calH)+\ln1/\delta}{N}}
  \end{equation}
  for some universal constant $c$.
\end{theorem}

\begin{proof}
  Let $d=\dVC(\calH)$ and $\calF= \{x\mapsto w(x)h(x):h\in\calH\}$. We just need to show that $\dP(\calF)\leq d$, then apply \cref{thm:pd}.

  Let $z_1,\dots,z_{d+1}\in\calX$ be distinct points.  Suppose $\calF$ pseudo-shatters this set, then $\exists t_1,\dots,t_{d+1}$ s.t. $\forall b_1,\dots,b_{d+1}\in\{0,1\}$ and for all $i$,
  \begin{alignat}{4}
     &            & \exists f\in\calF, &  & \quad & \1[f(z_i)\geq t_i] = b_i                   \\
     & \iff \quad & \exists h\in\calH, &  &       & \begin{cases}
                                                      h(z_i) \geq t_i/w(z_i) & \textrm{if $b_i=1$} \\
                                                      h(z_i) < t_i/w(z_i)    & \textrm{if $b_i=0$}
                                                    \end{cases}
    \\
     & \iff       & \exists h\in\calH, &  &       & h(z_i) = b_i,
  \end{alignat}
  where the third line follows from observing that we must have set the thresholds s.t.\ $t_i/w(z_i)\in (0,1]$, otherwise the inequality will always fail in one direction regardless of $h$, i.e., some configurations of $b_i$ will not be satisfied. Since the $z_i$'s are distinct, the above implies that $\calH$ shatters a set of size $(d+1)>\dVC(\calH)=d$, which is a contradiction, so $\dP(\calF) \leq d$.
\end{proof}

Our post-processings in \cref{thm:opt.fair,thm:sample} are linear multiclass classifiers (breaking ties to the smallest index), $\RR^d\rightarrow [L]$, parameterized by
\begin{equation}
 \calH_L  = \bigg\{  x\mapsto \min \! \bigg( \argmin_{i \in [L]} w^\T_i x \bigg) : w_1,\dots,w_L \in \RR^d \bigg\}.
\end{equation}
We bound the VC dimension of the above class of multiclass classifiers in one-versus-rest mode, i.e., the class of binary functions given by
\begin{equation}
  \calH_L^{\mathrm{ovr}, \ell} = \cbr*{ x\mapsto \1[h(x)=\ell] : h\in \calH_L }. \label{eq:lmc.ovr}
\end{equation}

\begin{lemma}\label{prop:vc.ovr}
  For all input dimension $d$, number of classes $L$, and $\ell\in[L]$, $\dVC(\calH^{\textnormal{ovr},\ell}_{L})\leq O(d \log L)$; see definition in \cref{eq:lmc.ovr}.
\end{lemma}

\begin{proof}
The proof proceeds by showing that $\calH_L^{\mathrm{ovr}, \ell}$ can be represented by \textit{feed-forward linear threshold networks}, then cites an existing result on the VC dimension of this class of networks \citep[Theorem 6.1]{anthony1999NeuralNetworkLearning}.

  Any $h\in \calH_L^{\mathrm{ovr}, \ell}$ can be written for some $w_1,\dots,w_L\in\RR^d$ as
  \begin{align}
     h(x) 
     & = \1[\min (\argmin\nolimits_{i} w_i^\T x)=\ell] \\
     & = \1\sbr*{\sum_{i<\ell}\1\sbr{ w_\ell^\T x > w_i^\T x } + \sum_{i> \ell}\1\sbr{ w_\ell^\T x \geq w_i^\T x } \geq L-1}          \\
     & = \1\sbr*{(\ell-1)-\sum_{i<\ell}\1\sbr{ w_\ell^\T x \leq w_i^\T x } + \sum_{i> \ell}\1\sbr{ w_\ell^\T x \geq w_i^\T x } \geq L-1} \\
     & = \1\sbr*{-\sum_{i<\ell}\1\sbr{ (w_i-w_\ell)^\T x \geq 0 } + \sum_{i> \ell}\1\sbr{(w_\ell-w_i)^\T x \geq 0 }  \geq L-\ell},
  \end{align}
  which is a two-layer feed-forward linear threshold network with
  $O(d)$
  variable weights and thresholds and
  $O(L)$ computation units (a.k.a.\ perceptrons/nodes), so its VC dimension is  $O(d \log L)$.
\end{proof}

Last but not least, note that the deterministic classifier $\hat h$ constructed in \cref{thm:sample} with weights $\hat w$ from solving the empirical problem $\LP(r,g,\widehat \PP_X,\calC,\alpha)$ may not be optimal on the empirical problem nor satisfies its constraints.  We want to bound the difference between the empirical loss of $\hat h$ and the optimal value of $\LP(r,g,\widehat \PP_X,\calC,\alpha)$, as well as violation of the empirical constraints by $\hat h$.

\begin{corollary}\label{cor:representation}
Under the same conditions as \cref{thm:sample}, denote by $\widehat{\mathrm{OPT}}$ the optimal value of $\LP(r,g,\widehat \PP_X,\calC,\alpha)$ and $(\hat\pi,\hat q)$ its minimizer.  Then $\forall N\geq \max_{k} 2 \ln(2K/\delta)/\Pr(Z_k=1)^2$, with probability at least $1-\delta$,
  \begin{align}
    \envert*{\widehat\E_X \sbr{ r(X, \hat h(X))} - \widehat{\mathrm{OPT}}} &\leq \frac{\enVert r_\infty \envert\calY^2}{N},\\
    \envert*{ \widehat\E_X \sbr*{\frac{g(X,k)}{\widehat\Pr(Z_k=1)} \1[\hat h(X)=y_c]  } - \hat q_c} &\leq\frac\alpha2 + \frac{\envert\calY}{N \Pr(Z_k=1)},\quad \forall c\in[C],\, k\in\calI_c.
  \end{align}
\end{corollary}

This follows from bounding the disagreements between the deterministic $\hat h$ and the randomized classifier $\hat \pi$ that minimizes the empirical problem on the samples $x_1,\dots,x_N$; disagreement occurs when $\hat\pi$ needs to randomize its output.

\begin{lemma}\label{lem:representation}
Under the same conditions as \cref{thm:sample}, denote the minimizer of $\LP(r,g,\widehat \PP_X,\calC,\alpha)$ by $(\hat\pi,\hat q)$. Then almost surely, for all $y\in\calY$,
  \begin{equation}
      \sum_{i=1}^N \envert*{\1[\hat h(x_i)=y] - \hat\pi(x_i,y)} \leq \envert\calY-1.
  \end{equation}
\end{lemma}

\begin{proof}
  The binary classifier $x\mapsto \1[\hat h(x)=y]$ is a linear multiclass classifier in one-versus-rest mode \cref{eq:lmc.ovr} on the features $u(x)\coloneqq(r(x),g(x))\in\RR^{\calY\times K}$:
  \begin{equation}
  \1[\hat h(x)=y] = \1\!\bigg[\min \!\bigg( \argmin_{y'\in\calY} \beta_{y'}^\T u(x) \bigg) = y \bigg], \label{eq:lem.c9.0}
  \end{equation}
   where we defined $\beta_y=(\be_{y}, \hat w_{y})$ with $(\hat w_y)_k= -(\sum_{c} \1[y_c=y, k\in\calI_c] {\hat \psi_{c,k}}/{\widehat \Pr(Z_k=1)})$ and $\hat\psi$ is the optimal dual values.

  As a consequence of complementary slackness in \cref{eq:1}, $\1[\hat h(x)=y]=\hat\pi(x,y)$ when $y$ is the unique ``best'' class that attains the $\argmin$ in \cref{eq:lem.c9.0}, or when $y$ is not a ``best'' class.  This means that disagreement occurs when the feature $u(x)$ lies on the decision boundary between $y$ and some other class $y'$, and therefore
  \begin{equation}
      \sum_{i=1}^N \envert*{\1[\hat h(x_i)=y] - \hat\pi(x_i,y)} \leq \sum_{i=1}^N \1\sbr*{\exists y'\neq y,\, (\beta_y - \beta_{y'})^\T u(x_i) = 0 }.
  \end{equation}

  Because the features $u(x_i)$ are samples of the push-forward $(r,g)\sharp\PP_X$, by \cref{ass:continuity}, for each $y'\neq y$, almost surely no more than one $u(x_i)$ can occupy the linear subspace given by $(\beta_y - \beta_{y'})$, which has two non-zero coordinates in the $\calY$-component.  So the number of disagreements is no more than $\envert\calY -1$.
\end{proof}

\begin{proof}[Proof of \cref{cor:representation}]
  By objective of \LPp,
  \begin{equation}
    \widehat{\mathrm{OPT}} =   \frac1N\sum_{i=1}^N\sum_{y\in\calY} r(x_i,y)\hat\pi(x_i, y),
  \end{equation}
  and by \cref{lem:representation}, almost surely,
  \begin{align}
    \envert*{\widehat\E_X \sbr{ r(X, \hat h(X))} - \widehat{\mathrm{OPT}}}
     &= \frac1N \envert*{\sum_{i=1}^N\sum_{y\in\calY} r(x_i,y) \rbr*{\1[\hat h(x_i)=y] - \hat\pi(x_i,y)}} \\
     &\leq \frac1N\enVert  r_\infty   \sum_{y\in\calY} \sum_{i=1}^N \envert*{\1[\hat h(x_i)=y] - \hat\pi(x_i,y)} \\
     &\leq \frac1N\enVert r_\infty \envert\calY(\envert\calY-1).
  \end{align}

  Similarly, by the constraint of \LPp, for all $c\in[C]$ and $k\in\calI_c$,
  \begin{equation}
  \envert*{\frac1N\sum_{i=1}^N\frac{g(x_i,k)}{\widehat\Pr(Z_k=1)}\hat\pi(x_i,y_c) - \hat q_c  } \leq \frac\alpha2,
  \end{equation}
  and by \cref{lem:representation}, almost surely,
  \begin{align}
    \envert*{ \widehat\E_X \sbr*{\frac{g(X,k)}{\widehat\Pr(Z_k=1)} \1[\hat h(X)=y_c]  } - \hat q_c}  
    &\leq \frac\alpha2 + \frac1N\envert*{\sum_{i=1}^N\frac{g(x_i,k)}{\widehat\Pr(Z_k=1)}\rbr*{\1[\hat h(x_i) = y_c] - \hat\pi(x_i,y_c)} } \\
    &\leq \frac\alpha2 + \frac1{N \widehat\Pr(Z_k=1)}\sum_{i=1}^N\envert*{\1[\hat h(x_i) = y_c] - \hat\pi(x_i,y_c)}  \\
    &\leq \frac\alpha2 + \frac1{N \widehat\Pr(Z_k=1)}(\envert\calY-1) \\
    &\leq \frac\alpha2 + \frac2{N \Pr(Z_k=1)}(\envert\calY-1);
  \end{align}
  the last inequality holds for all $N\geq \max_{k} 2\ln(2K/\delta)/\Pr(Z_k=1)^2$  w.p.\ at least $1-\delta$, under which
  \begin{equation}
  \envert*{\widehat\Pr(Z_k=1) - \Pr(Z_k=1)}\leq \sqrt{\frac{1}{2N}\ln\frac{2K}\delta}\leq \frac{\Pr(Z_k=1)}2 \label{eq:pop.emp}
\end{equation}
 by Hoeffding's inequality and a union bound, which implies $1/\widehat\Pr(Z_k=1)\leq 2/\Pr(Z_k=1)$.
\end{proof}

With all the tools above, we can now prove \cref{thm:sample}.

\begin{proof}[Proof of \cref{thm:sample}]
  We first bound the violation of fairness constraints, then the excess risk.

  \paragraph{Constraint Violation.}
  Denote  a minimizer of the empirical $\LPp(r,g,\widehat \PP_X,\calC,\alpha)$ by $(\hat\pi,\hat q)$, then for all $c\in[C]$ and $k\in\calI_c$, w.p.\ at least $1-\delta$,
  \begin{align}
    \MoveEqLeft\envert*{\Pr(\hat h(X)=y_c\mid Z_k=1) - \hat q_c} \\
    &=  \envert*{ \E_X \sbr*{\frac{g(X,k)}{\Pr(Z_k=1)} \1[\hat h(X)=y_c]  } - \hat q_c} \\
    &\leq \begin{multlined}[t]
    \envert*{ \widehat\E_X \sbr*{\frac{g(X,k)}{\widehat\Pr(Z_k=1)} \1[\hat h(X)=y_c]  } - \hat q_c} \\
      + \envert*{ \widehat\E_X \sbr*{\frac{g(X,k)}{\widehat\Pr(Z_k=1)} \1[\hat h(X)=y_c]  } - \widehat\E_X \sbr*{\frac{g(X,k)}{\Pr(Z_k=1)} \1[\hat h(X)=y_c]  }}\\
      + \envert*{  \widehat\E_X \sbr*{\frac{g(X,k)}{\Pr(Z_k=1)} \1[\hat h(X)=y_c]  } - \E_X \sbr*{\frac{g(X,k)}{\Pr(Z_k=1)} \1[\hat h(X)=y_c]  }}\end{multlined} \\
    &\leq \begin{multlined}[t] \frac\alpha 2 + \frac{\envert\calY}{N\Pr(Z_k=1)}    + \widehat \E_X[g(X,k)]\envert*{ \frac{1}{\widehat\Pr(Z_k=1)} -\frac{1}{\Pr(Z_k=1)} }  \\
     + O\rbr*{\frac{1}{\Pr(Z_k=1)} \sqrt{\frac{\envert\calY K \log \envert\calY + \ln K/\delta}{N}} } \end{multlined} \\
    &\leq \frac\alpha 2 + \frac{\envert\calY}{N\Pr(Z_k=1)}    + O\rbr*{ \frac{1}{\Pr(Z_k=1)} \sqrt{\frac{\ln K/\delta}{N}} + \frac{1}{\Pr(Z_k=1)} \sqrt{\frac{\envert\calY K \log \envert\calY + \ln K/\delta}{N}} },
  \end{align}
  where the second inequality is by \cref{cor:representation} and \cref{thm:im.vc} with the VC dimension in \cref{prop:vc.ovr}, and the last inequality follows from \cref{eq:pop.emp}.    So, when $N\geq \Omega(\envert\calY)$,
\begin{align}
  \MoveEqLeft \max_{k,k'\in\calI_c}\envert*{\Pr(\hat h(X)=y_c\mid Z_k=1) - \Pr(\hat h(X)=y_c\mid Z_{k'}=1)  } \\
  &\leq \max_{k\in\calI_c} 2\envert*{\Pr(\hat h(X)=y_c\mid Z_k=1) - \hat q_c } \\
  &\leq  \alpha + O\rbr*{ \max_{k\in[K]}\frac{1}{\Pr(Z_k=1)^2} \sqrt{\frac{\envert\calY K \log \envert\calY + \ln K/\delta}{N}} },
\end{align}
and this concludes the bound on constraint violation.

  Before continuing to bounding the excess risk, we similarly bound the constraint violation of the optimal fair classifier $h$ on the empirical problem $\LP(r,g,\widehat\PP_X,\calC,\alpha)$, which will be used later (in the above, we have bounded that of the empirical fair classifier $\hat h$ on the population problem $\LP(r,g,\PP_X,\calC,\alpha)$).  Let $(\pi,q)$ and $(\phi,\psi)$ denote the optimal primal and dual values of $\LP(r,g,\widehat\PP_X,\calC,\alpha)$, respectively.  Note that the optimal fair classifier $h$, given in \cref{thm:opt.fair}, is equal to $\pi$ almost everywhere under \cref{ass:uniqueness}.
  
  With probability at least $1-\delta$, for all $c\in[C]$ and $k\in\calI_c$,
    \begin{align}
    \MoveEqLeft\envert*{\widehat\Pr(h(X)=y_c\mid Z_k=1) -  q_c} \\
    &\leq \begin{multlined}[t]\envert*{ \E_X \sbr*{\frac{g(X,k)}{\Pr(Z_k=1)} \1[ h(X)=y_c]  } - q_c} \\
    + \envert*{ \E_X \sbr*{\frac{g(X,k)}{\Pr(Z_k=1)} \1[ h(X)=y_c]  } -\E_X \sbr*{\frac{g(X,k)}{\widehat\Pr(Z_k=1)} \1[ h(X)=y_c]  }}\\
     + \envert*{  \E_X \sbr*{\frac{g(X,k)}{\widehat\Pr(Z_k=1)} \1[ h(X)=y_c]  } - \widehat\E_X \sbr*{\frac{g(X,k)}{\widehat\Pr(Z_k=1)} \1[ h(X)=y_c]  }} \end{multlined} \\
    &\leq \frac\alpha 2 +  O\rbr*{ \frac{1}{\Pr(Z_k=1)} \sqrt{\frac{\ln K/\delta}{N}} + \frac{1}{\Pr(Z_k=1)} \sqrt{\frac{\ln K/\delta}{N}} }
  \end{align}
   by Hoeffding's inequality, since $h$ does not depend on the samples. So the constraint violation of $h$ on the empirical problem is
\begin{align}
  \max_{k,k'\in\calI_c}\envert*{\widehat\Pr( h(X)=y_c\mid Z_k=1) - \widehat\Pr( h(X)=y_c\mid Z_{k'}=1)  } 
  &\leq  \alpha + O\rbr*{ \max_{k\in[K]} \frac{1}{\Pr(Z_k=1)} \sqrt{\frac{\ln K/\delta}{N}}} \\
  &\eqqcolon \alpha + B_1.
\end{align}

  \paragraph{Excess Risk.}
We begin by constructing a randomized classifier $\tilde h$ in terms of $h$ that satisfies the fairness constraints on the empirical problem  (similar to the proof of \cref{thm:sensitivity} Part~2), given by
  \begin{equation}
    \tilde h =\begin{cases}
      h & \textrm{w.p. $1-\beta$} \\
      1                    & \textrm{w.p. $\beta$},  \\
    \end{cases} \qquad\text{where}\quad \beta = \min\rbr*{1-\alpha, \frac{B_1}{\alpha+B_1}}\in[0,1].
  \end{equation}
  Then we can bound the optimal value of the empirical problem $\widehat{\mathrm{OPT}}$ by the value of $\tilde h$:
  \begin{align}\label{eq:sample.3}
    \widehat{\mathrm{OPT}}
    \leq \widehat\E_X\E_{\tilde h} [r(X,\tilde h(X))]\leq (1-\beta) \widehat\E_X [r(X,h(X))] + \beta\enVert r_\infty,
  \end{align}
  and we note that $\beta\leq B_1/\alpha$.
  
   Next, note that
  \begin{align}
    \envert*{ \E_X [r(X,\hat h(X))] - \widehat\E_X [r(X,\hat h(X))] } 
    &\leq \sum_{y\in\calY} \envert*{ \E_X\sbr*{r(X,y)\1[\hat h(X)=y]} -  \widehat\E_X\sbr*{r(X,y)\1[\hat h(X)=y] }} \\
    &\leq \envert\calY \enVert r_\infty O\rbr*{ \sqrt{\frac{\envert\calY K \log \envert\calY + \ln K/\delta}{N}} } \\
    &\eqqcolon B_2 \label{eq:sample.4}
  \end{align}
  by \cref{thm:im.vc} with the VC dimension in \cref{prop:vc.ovr}.
  
  Then, putting everything together,
  \begin{align}
    \MoveEqLeft\E_X [r(X, \hat h(X))] \\
     & \leq \widehat\E_X [r(X, \hat h(X))] + B_2                                                                                 &  & \textrm{by \cref{eq:sample.4}}        \\
     & \leq \widehat{\mathrm{OPT}} + \frac{\enVert r_\infty\envert\calY^2}N  + B_2   &  & \textrm{by \cref{cor:representation}} \\
     & \leq \widehat\E_X [r(X, h(X))] + \frac{\enVert r_\infty B_1}{\alpha}  + \frac{\enVert r_\infty\envert\calY^2}N  + B_2 &  & \textrm{by \cref{eq:sample.3}}        \\
     & \leq \E_X [r(X, h(X))] + \frac{\enVert r_\infty B_1}{\alpha}  + \frac{\enVert r_\infty\envert\calY^2}N  + 2B_2        &  & \textrm{by \cref{eq:sample.4}}        \\
     & =  \mathrm{OPT} + \frac{\enVert r_\infty B_1}{\alpha}  + \frac{\enVert r_\infty\envert\calY^2}N  + 2B_2                                 &  & \textrm{by \cref{thm:opt.fair}.}
  \end{align}

  We may conclude by taking a final union bound over all events considered above.
\end{proof}

\begin{table}[p]
  \caption{Wall clock running time (in seconds) of the post-processing stage of LinearPost (i.e., solving the \protect\LP\ on \cref{ln:alg.lp}) for the results in \cref{fig:exp.adult,fig:exp.compas,fig:exp.acsincome2,fig:exp.acsincome5,fig:exp.biasbios}, with the Gurobi optimizer under $\alpha=0.001$. Average of five random seeds on an Amazon EC2 \texttt{c6i.16xlarge} instance.}
  \label{tab:runtime}
  \centering
  \scalebox{0.8}{%
    \begin{tabular}{lC{2cm}C{2cm}C{2cm}}
      \toprule
       Backbone           & SP                                                                                   & EOpp  & EO      \\
      \midrule
      \textit{Adult}      & \multicolumn{3}{l}{$\envert\calA=2$, $\envert\calY=2$, $n_\textrm{postproc}=17095$}                    \\
      \midrule
      Log.\ reg.          & 0.64                                                                                 & 0.57  & 0.77    \\
      GBDT                & 0.63                                                                                 & 0.57  & 0.74    \\
      \midrule
      \textit{COMPAS}     & \multicolumn{3}{l}{$\envert\calA=2$, $\envert\calY=2$, $n_\textrm{postproc}=1847$}                     \\
      \midrule
      Log.\ reg.          & 0.08                                                                                 & 0.07  & 0.10    \\
      GBDT                & 0.08                                                                                 & 0.07  & 0.09    \\
      \midrule
      \textit{ACSIncome2} & \multicolumn{3}{l}{$\envert\calA=2$, $\envert\calY=2$, $n_\textrm{postproc}=116515$}                   \\
      \midrule
      Log.\ reg.          & 4.90                                                                                 & 4.26  & 5.73    \\
      GBDT                & 5.14                                                                                 & 4.20  & 6.21    \\
      MLP                 & 4.74                                                                                 & 4.18  & 5.72    \\
      \midrule
      \textit{ACSIncome5} & \multicolumn{3}{l}{$\envert\calA=5$, $\envert\calY=5$, $n_\textrm{postproc}=116515$}                   \\
      \midrule
      Log.\ reg.          & 18.23                                                                                & 19.76 & 76.75   \\
      GBDT                & 15.67                                                                                & 17.25 & 88.48   \\
      MLP                 & 15.86                                                                                & 17.56 & 82.09   \\
      \midrule
      \textit{BiasBios}   & \multicolumn{3}{l}{$\envert\calA=2$, $\envert\calY=28$, $n_\textrm{postproc}=55080$}                   \\
      \midrule
      BERT                & 38.98                                                                                & 41.46 & 2826.37 \\
      \bottomrule
    \end{tabular}
  }%
\end{table}

\begin{figure}[p]
  \centering
  \includegraphics[width=0.8\linewidth]{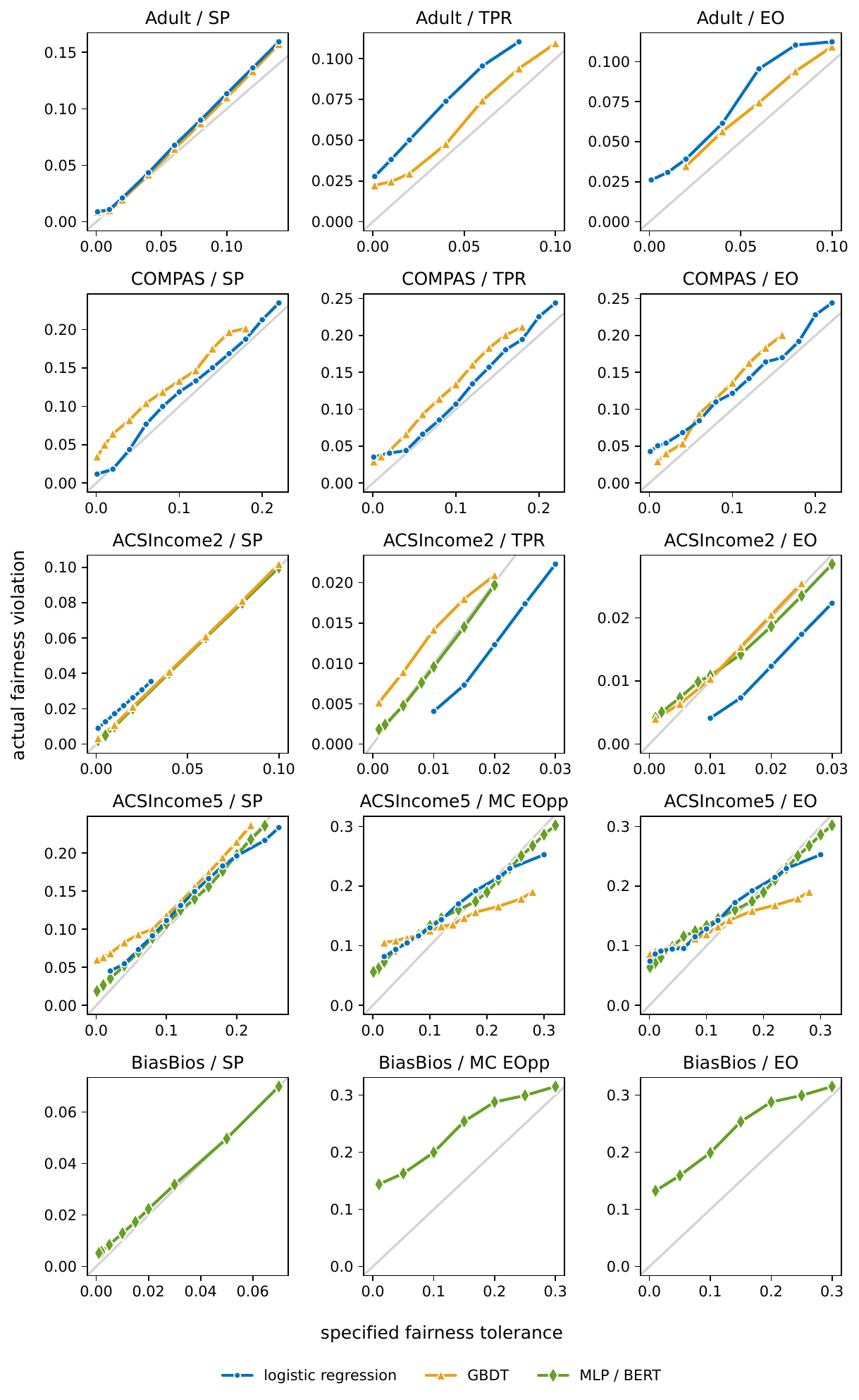}
  \caption{Specified fairness tolerance $\alpha$ vs.\ the actual achieved level of fairness for LinearPost.}
  \label{fig:alpha}
\end{figure}

\begin{figure}[p]
  \centering
  \includegraphics[width=1\linewidth]{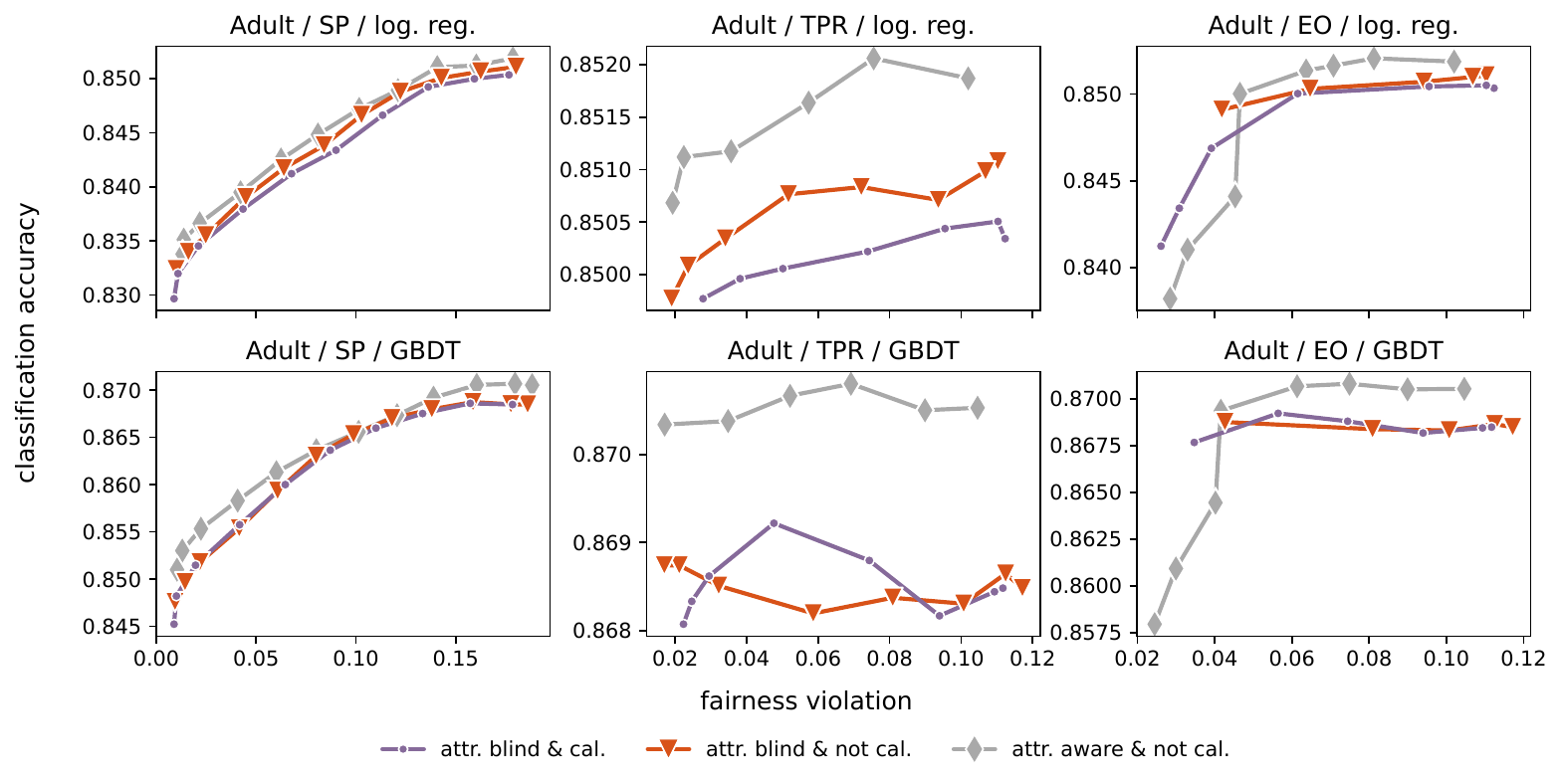}
  \caption{Ablation studies for LinearPost on Adult, comparing post-processing results on calibrated vs.\ uncalibrated predictors in the attribute-blind setting, and on uncalibrated attribute-aware vs.\ attribute-blind predictors.  Average of five random seeds.}
  \label{fig:abl.adult}
\end{figure}

\begin{figure}[p]
  \centering
  \includegraphics[width=1\linewidth]{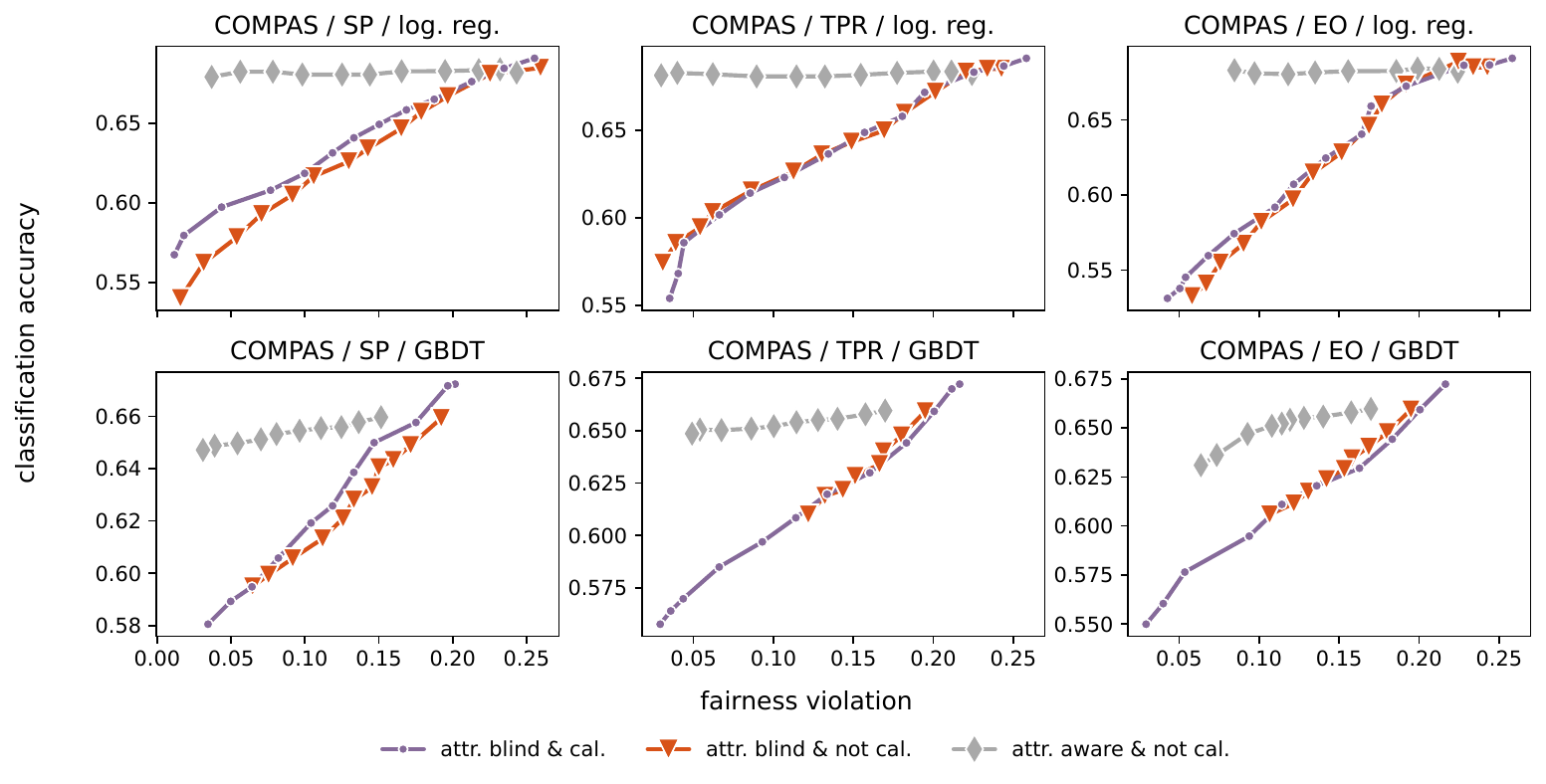}
  \caption{Ablation studies for LinearPost on COMPAS\@.}
  \label{fig:abl.compas}
\end{figure}

\begin{figure}[p]
  \centering
  \includegraphics[width=1\linewidth]{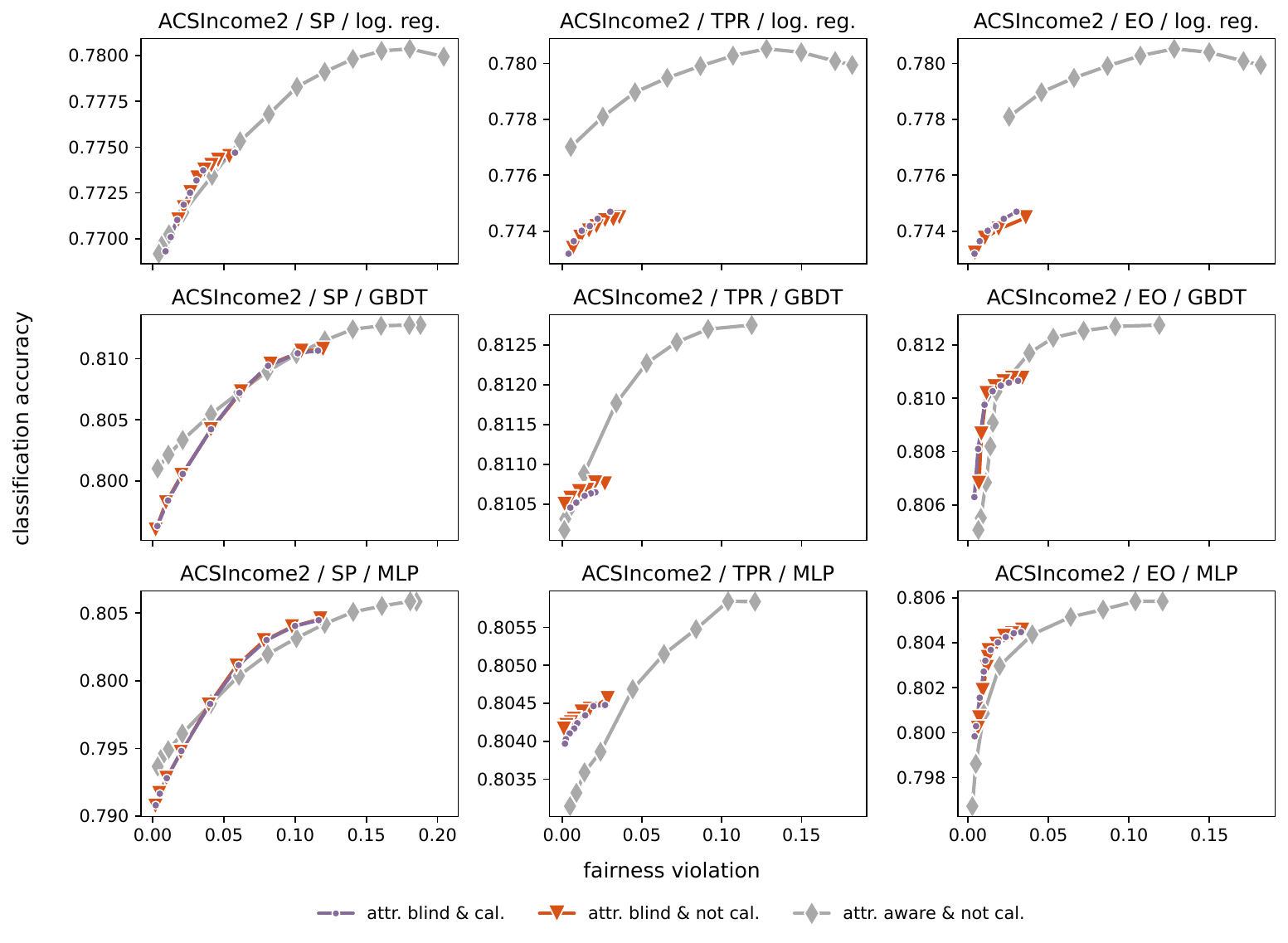}
  \caption{Ablation studies for LinearPost on ACSIncome2.}
  \label{fig:abl.acsincome2}
\end{figure}

\begin{figure}[p]
  \centering
  \includegraphics[width=1\linewidth]{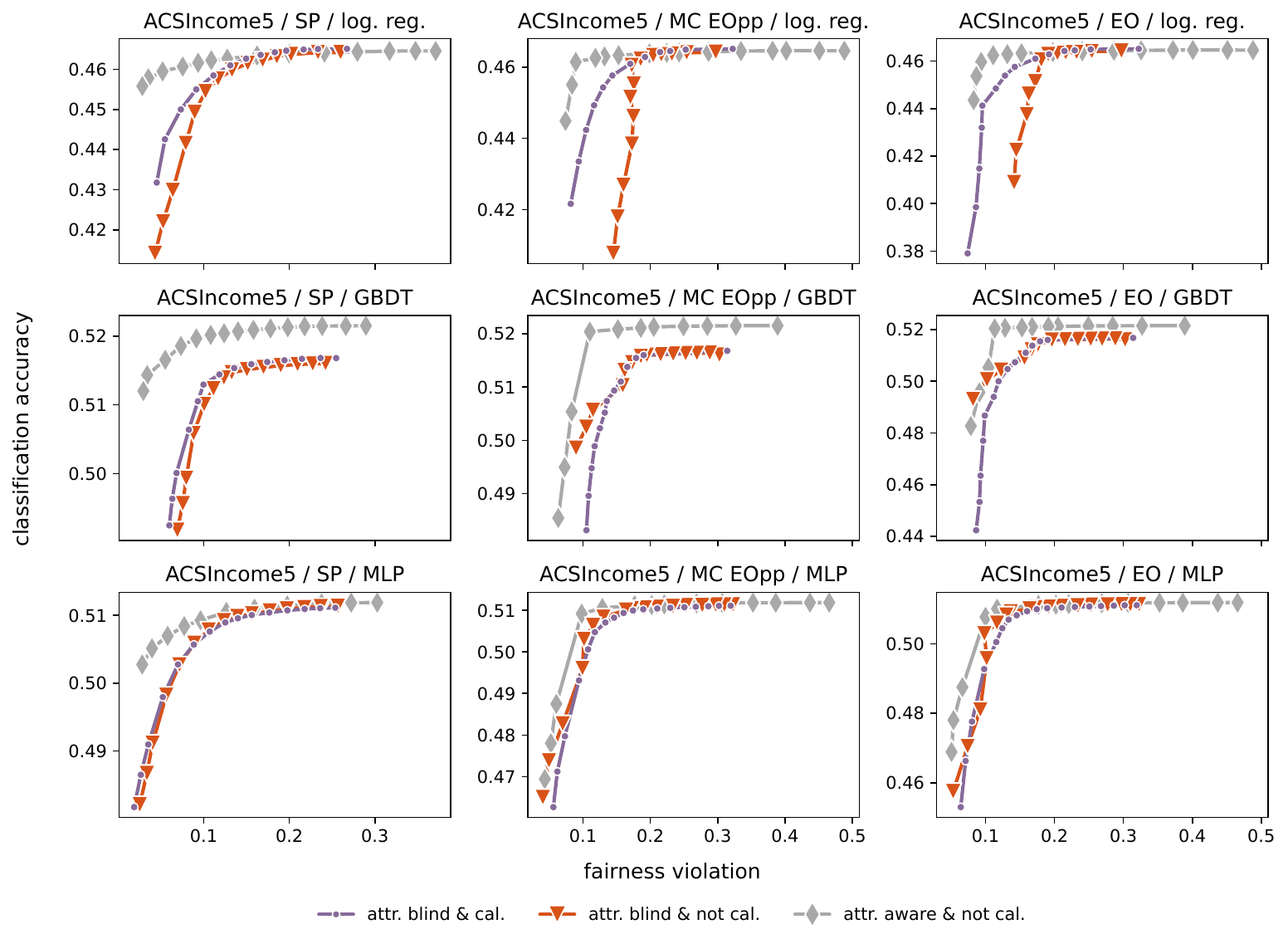}
  \caption{Ablation studies for LinearPost on ACSIncome5.}
  \label{fig:abl.acsincome5}
\end{figure}

\begin{figure}[p]
  \centering
  \includegraphics[width=1\linewidth]{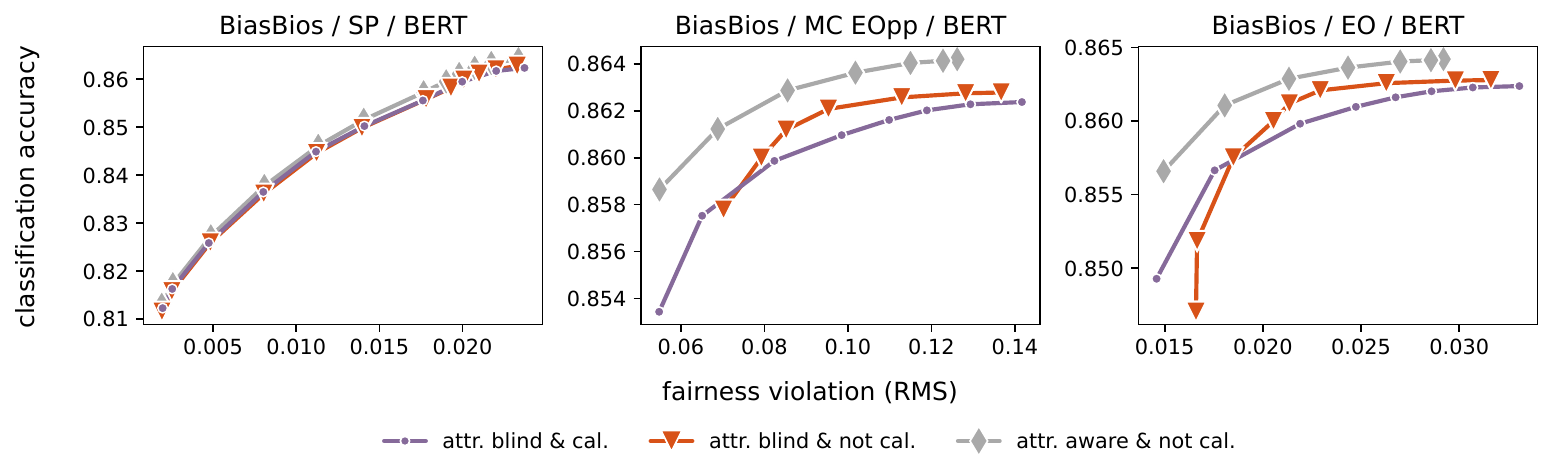}
  \caption{Ablation studies for LinearPost on BiasBios.}
  \label{fig:abl.biasbios}
\end{figure}

\end{document}